\documentclass[10pt,journal,epsfig]{IEEEtran}

\usepackage[dvips]{graphicx}
\usepackage{graphicx}
\usepackage{amssymb}
\usepackage{cite}
\usepackage{amsmath}
\usepackage{algorithm}
\usepackage{algorithmic}
\usepackage{multirow}
\usepackage{xcolor}
\usepackage{booktabs}
\usepackage{url}
\usepackage{subfig}

\begin{document}


\title{Fast Low-Rank Bayesian Matrix Completion with Hierarchical Gaussian
Prior Models}

\author{Linxiao Yang, Jun Fang, Huiping Duan, Hongbin Li,~\IEEEmembership{Senior
Member,~IEEE}, and Bing Zeng,~\IEEEmembership{Fellow,~IEEE}
\thanks{Linxiao Yang and Jun Fang are with the National Key Laboratory
of Science and Technology on Communications, University of
Electronic Science and Technology of China, Chengdu 611731, China,
Email: JunFang@uestc.edu.cn}
\thanks{Huiping Duan and Bing Zeng are with the School of Electronic Engineering,
University of Electronic Science and Technology of China, Chengdu
611731, China, Emails: huipingduan@uestc.edu.cn;
eezeng@uestc.edu.cn}
\thanks{Hongbin Li is
with the Department of Electrical and Computer Engineering,
Stevens Institute of Technology, Hoboken, NJ 07030, USA, E-mail:
Hongbin.Li@stevens.edu}
\thanks{This work was supported in part by the National Science
Foundation of China under Grant 61522104.}}

\maketitle

\begin{abstract}
The problem of low-rank matrix completion is considered in this
paper. To exploit the underlying low-rank structure of the data
matrix, we propose a hierarchical Gaussian prior model, where
columns of the low-rank matrix are assumed to follow a Gaussian
distribution with zero mean and a common precision matrix, and a
Wishart distribution is specified as a hyperprior over the
precision matrix. We show that such a hierarchical Gaussian prior
has the potential to encourage a low-rank solution. Based on the
proposed hierarchical prior model, we develop a variational
Bayesian matrix completion method which embeds the generalized
approximate massage passing (GAMP) technique to circumvent
cumbersome matrix inverse operations. Simulation results show that
our proposed method demonstrates superiority over some
state-of-the-art matrix completion methods.
\end{abstract}






\section{Introduction}
The problem of recovering a partially observed matrix, which is
referred to as matrix completion, arises in a variety of
applications, including recommender systems
\cite{KaratzoglouAmatriain10,XiongChen10,AdomaviciusTuzhilin2011},
genotype prediction \cite{JiangMa2016,ChiZhou2013}, image
classification \cite{CabralTorre2015,LuoLiu2015}, network traffic
prediction \cite{YeChen2017}, and image imputation
\cite{HeSun2014}. Low-rank matrix completion, which is empowered
by the fact that many real-world data lie in an intrinsically low
dimensional subspace, has attracted much attention over the past
few years. Mathematically, a canonical form of the low-rank matrix
completion problem can be presented as
\begin{align}
\min_{\boldsymbol{X}} \quad&\text{rank}(\boldsymbol{X})\nonumber\\
\text{s.t.}\quad
&\boldsymbol{Y}=\boldsymbol{\Omega}\ast\boldsymbol{X} \label{prob}
\end{align}
where $\boldsymbol{X}\in\mathbb{R}^{M\times N}$ is an unknown
low-rank matrix, $\boldsymbol{\Omega}\in\{0,1\}^{M\times N}$ is a
binary matrix that indicates which entries of $\boldsymbol{X}$ are
observed, $\ast$ denotes the Hadamard product, and
$\boldsymbol{Y}\in\mathbb{R}^{M\times N}$ is the observed matrix.
It has been shown that the low-rank matrix $\boldsymbol{X}$ can be
exactly recovered from (\ref{prob}) under some mild conditions
\cite{CandesRecht2009}. Nevertheless, minimizing the rank of a
matrix is an NP-hard problem and no known polynomial-time
algorithms exist. To overcome this difficulty, alternative
low-rank promoting functionals were proposed. Among them, the most
popular alternative is the nuclear norm which is defined as the
sum of the singular values of a matrix. Replacing the rank
function with the nuclear norm yields the following convex
optimization problem
\begin{align}
\min_{\boldsymbol{X}} \quad&\|\boldsymbol{X}\|_{*}\nonumber\\
\text{s.t.}\quad
&\boldsymbol{Y}=\boldsymbol{\Omega}\ast\boldsymbol{X} \label{opt}
\end{align}
It was proved that the nuclear norm is the tightest convex
envelope of the matrix rank, and the theoretical recovery
guarantee for (\ref{opt}) under both noiseless and noisy cases was
provided in
\cite{CandesRecht2009,RechtFazel2010,FazelCandes2008,CandesPlan2011}.
To solve (\ref{opt}), a number of computationally efficient
methods were developed. A well-known method is the singular value
thresholding method which was proposed in \cite{CaiCandes2010}.
Another efficient method was proposed in \cite{LinChen2010}, in
which an augmented Lagrange multiplier technique was employed.
Apart from convex relaxation, non-convex surrogate functions, such
as the log-determinant function, were also introduced to replace
the rank function
\cite{MohanFazel2010,Wipf2014,KangKang2015,YangXie2016}.
Non-convex methods usually claim better recovery performance,
since non-convex surrogate functions behaves more like the rank
function than the nuclear norm. It is noted that for both convex
methods and non-convex methods, one need to meticulously select
some regularization parameters to properly control the tradeoff
between the matrix rank and the data fitting error when noise is
involved. However, due to the lack of the knowledge of the noise
variance and the rank, it is usually difficult to determine
appropriate regularization parameters.






Another important class of low-rank matrix completion methods are
Bayesian methods
\cite{BabacanLuessi2012,ZhouWang2010,ParkerSchniter2014a,ParkerSchniter2014b,XinWang2016},
which model the problem in a Bayesian framework and have the
ability to achieve automatic balance between the low-rankness and
the fitting error. Specifically, in \cite{BabacanLuessi2012}, the
low-rank matrix is expressed as a product of two factor matrices,
i.e. $\boldsymbol{X}=\boldsymbol{A}\boldsymbol{B}^T$, and the
matrix completion problem is translated to searching for these two
factor matrices $\boldsymbol{A}$ and $\boldsymbol{B}$. To
encourage a low-rank solution, sparsity-promoting priors
\cite{ZhangRao11} are placed on the columns of two factor
matrices, which aims to promote structured-sparse factor matrices
with only a few non-zero columns, and in turn leads to a low-rank
matrix $\boldsymbol{X}$. Nevertheless, this Bayesian method
updates the factor matrices in a row-by-row fashion and needs to
perform a number of matrix inverse operations at each iteration.
To address this issue, a bilinear generalized approximate message
passing (GAMP) method was developed to learn the two factor
matrices $\boldsymbol{A}$ and $\boldsymbol{B}$
\cite{ParkerSchniter2014a,ParkerSchniter2014b}, without involving
any matrix inverse operations. This method, however, cannot
automatically determine the matrix rank and needs to try out all
possible values of the rank. Recently, a new Bayesian prior model
was proposed in \cite{XinWang2016}, in which columns of the
low-rank matrix $\boldsymbol{X}$ follow a zero mean Gaussian
distribution with an unknown deterministic covariance matrix that
can be estimated via Type II maximum likelihood. It was shown that
maximizing the marginal likelihood function yields a low-rank
covariance matrix, which implies that the prior model has the
ability to promote a low-rank solution. A major drawback of this
method is that it requires to perform an inverse of an $MN\times
MN$ matrix at each iteration, and thus has a cubic complexity in
terms of the problem size. This high computational cost prohibits
its application to many practical problems.


In this paper, we develop a new Bayesian method for low-rank
matrix completion. To exploit the underlying low-rank structure of
the data matrix, a low-rank promoting hierarchical Gaussian prior
model is proposed. Specifically, columns of the low-rank matrix
$\boldsymbol{X}$ are assumed to be mutually independent and follow
a common Gaussian distribution with zero mean and a precision
matrix. The precision matrix is treated as a random parameter,
with a Wishart distribution specified as a hyperprior over it. We
show that such a hierarchical Gaussian prior model has the
potential to encourage a low-rank solution. The GAMP technique is
employed and embedded in the variational Bayesian (VB) inference,
which results in an efficient VB-GAMP algorithm for matrix
completion. Note that due to the non-factorizable form of the
prior distribution, the GAMP technique cannot be directly used. To
address this issue, we construct a carefully devised surrogate
problem whose posterior distribution is exactly the one required
for VB inference. Meanwhile, the surrogate problem has
factorizable prior and noise distributions such that the GAMP can
be directly applied to obtain an approximate posterior
distribution. Such a trick helps achieve a substantial
computational complexity reduction, and makes it possible to
successfully apply the proposed method to solve large-scale matrix
completion problems.

The rest of the paper is organized as follows. In Section
\ref{sec:prior-model}, we introduce a hierarchical Gaussian prior
model for low-rank matrix completion. Based on this hierarchical
model, a variational Bayesian method is developed in Section
\ref{sec:VB}. In Section \ref{sec:GAMP-VB}, a GAMP-VB method is
proposed to reduce the computational complexity of the proposed
algorithm. Simulation results are provided in Section
\ref{sec:simulation-results}, followed by concluding remarks in
Section \ref{sec:conclusions}.

\begin{figure}[t]
   \centering
   \includegraphics [width=200pt]{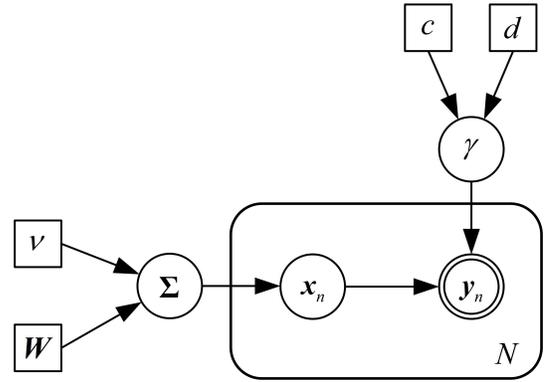}
   \caption{Proposed low-rank promoting hierarchical Gaussian prior model}
   \label{fig:graphical-model}
\end{figure}

\section{Bayesian Modeling} \label{sec:prior-model}
In the presence of noise, the canonical form of the matrix
completion problem can be formulated as
\begin{align}
\min_{\boldsymbol{X}} \quad&\text{rank}(\boldsymbol{X})\nonumber\\
\text{s.t.}\quad
&\boldsymbol{Y}=\boldsymbol{\Omega}\ast(\boldsymbol{X}+\boldsymbol{E})
\label{opt1}
\end{align}
where $\boldsymbol{E}$ denotes the additive noise, and
$\boldsymbol{\Omega}\in\{0,1\}^{M\times N}$ is a binary matrix
that indicates which entries are observed. Without loss of
generality, we assume $M\le N$. As indicated earlier, minimizing
the rank of a matrix is an NP-hard problem. In this paper, we
consider modeling the matrix completion problem within a Bayesian
framework.

We assume entries of $\boldsymbol{E}$ are independent and
identically distributed (i.i.d.) random variables following a
Gaussian distribution with zero mean and variance $\gamma^{-1}$.
To learn $\gamma$, a Gamma hyperprior is placed over $\gamma$,
i.e.
\begin{align}
p(\gamma)=\text{Gamma}(\gamma|a,b)=\Gamma(a)^{-1}b^a\gamma^{a-1}e^{-b\gamma}
\end{align}
where $\Gamma(a)=\int_{0}^{\infty}t^{a-1}e^{-t}dt$ is the Gamma
function. The parameters $a$ and $b$ are set to small values, e.g.
$10^{-8}$, which makes the Gamma distribution a non-informative
prior.

To promote the low-rankness of $\boldsymbol{X}$, we propose a
two-layer hierarchical Gaussian prior model (see Fig.
\ref{fig:graphical-model}). Specifically, in the first layer, the
columns of $\boldsymbol{X}$ are assumed mutually independent and
follow a common Gaussian distribution:
\begin{align}
p(\boldsymbol{X}|\boldsymbol{\Sigma})=\prod\limits_{n=1}^{N}p(\boldsymbol{x}_n|\boldsymbol{\Sigma})
=\prod\limits_{n=1}^{N}\mathcal{N}(\boldsymbol{x}_n|\boldsymbol{0},\boldsymbol{\Sigma}^{-1})\label{x-prior}
\end{align}
where $\boldsymbol{x}_n$ denotes the $n$th column of
$\boldsymbol{X}$, and $\boldsymbol{\Sigma}\in\mathbb{R}^{M\times
M}$ is the precision matrix. The second layer specifies a Wishart
distribution as a hyperprior over the precision matrix
$\boldsymbol{\Sigma}$:
\begin{align}
p(\boldsymbol{\Sigma}) \propto&
|\boldsymbol{\Sigma}|^{\frac{\nu-M-1}{2}}\exp(-\frac{1}{2}
    \text{tr}(\boldsymbol{W}^{-1}\boldsymbol{\Sigma}))
\end{align}
where $\nu$ and $\boldsymbol{W}\in\mathbb{R}^{M\times M}$ denote
the degrees of freedom and the scale matrix of the Wishart
distribution, respectively. Note that the constraint $\nu>M-1$ for
the standard Wishart distribution can be relaxed to $\nu>0$ if an
improper prior is allowed, e.g. \cite{PhilipMorris2000}. In
Bayesian inference, improper prior distributes can often be used
provided that the corresponding posterior distribution can be
correctly normalized \cite{Bishop07}.


The Gaussian-inverse Wishart prior has the potential to encourage
a low-rank solution. To illustrate this low-rankness promoting
property, we integrate out the precision matrix
$\boldsymbol{\Sigma}$ and obtain the marginal distribution of
$\boldsymbol{X}$ as (details of the derivation can be found in
Appendix \ref{appB})
\begin{align}
p(\boldsymbol{X}) =&\int
\prod\limits_{n=1}^{N}p(\boldsymbol{x}_n|\boldsymbol{\Sigma})p(\boldsymbol{\Sigma})d\boldsymbol{\Sigma}\nonumber\\
\propto&
|\boldsymbol{W}^{-1}+\boldsymbol{X}\boldsymbol{X}^T|^{-\frac{\nu+N}{2}}\label{X-marginal}
\end{align}
From (\ref{X-marginal}), we have
\begin{align}
\log
p(\boldsymbol{X})\propto-\log|\boldsymbol{X}\boldsymbol{X}^T+\boldsymbol{W}^{-1}|
\end{align}
If we choose $\boldsymbol{W}=\epsilon^{-1}\boldsymbol{I}$, and let
$\epsilon$ be a small positive value, the log-marginal
distribution becomes
\begin{align}
\log p(\boldsymbol{X})\propto&
-\log|\boldsymbol{X}\boldsymbol{X}^T+\epsilon\boldsymbol{I}|
\nonumber\\
=& -\sum_{m=1}^{M}\log(\lambda_m+\epsilon)
\end{align}
where $\lambda_m$ denotes the $m$th eigenvalue of
$\boldsymbol{X}\boldsymbol{X}^T$. Clearly, in this case, the prior
$p(\boldsymbol{X})$ encourages a low-rank solution
$\boldsymbol{X}$. This is because maximizing the prior
distribution $p(\boldsymbol{X})$ is equivalent to minimizing
$\sum_{m=1}^{M}\log(\lambda_m+\epsilon)$ with respect to
$\{\lambda_m\}$. It is well known that the log-sum function
$\sum_{m=1}^{M}\log(\lambda_m+\epsilon)$ is an effective
sparsity-promoting functional which encourages a sparse solution
of $\{\lambda_m\}$ \cite{CandesWakin08,ShenFang13,FuHuang2015}. As
a result, the resulting matrix $\boldsymbol{X}$ has a low-rank
structure.

In addition to $\boldsymbol{W}=\epsilon^{-1}\boldsymbol{I}$, the
parameter $\boldsymbol{W}$ can otherwise be devised in order to
exploit additional prior knowledge about $\boldsymbol{X}$. For
example, in some applications such as image inpainting, there is a
spatial correlation among neighboring coefficients of
$\boldsymbol{x}_n$. To capture the smoothness between neighboring
coefficients, $\boldsymbol{W}$ can be set as \cite{ChenMolina2015}
\begin{align}
\boldsymbol{W}=\boldsymbol{F}^T\boldsymbol{F} \label{W-1}
\end{align}
where $\boldsymbol{F}\in\mathbb{R}^{M\times M}$ is a second-order
difference operator with its $(i,j)$th entry given by
\begin{align}
f_{i,j}=
\begin{cases}
-2,\quad &i=j\\
1,&|i-j|=1\\
0,&\text{else}
\end{cases}
\end{align}
Another choice of $\boldsymbol{W}$ to promote a smooth solution is
the Laplacian matrix \cite{ShumanNarang2013}, i.e.
\begin{align}
    \boldsymbol{W} = \boldsymbol{D}-\boldsymbol{A}+\hat{\epsilon}\boldsymbol{I}\label{W-2}
\end{align}
where $\boldsymbol{A}$ is the adjacency matrix of a graph with its
entries given by
\begin{align}
    a_{ij}=\exp\left(-\frac{|i-j|^2}{\theta^2}\right)
\end{align}
$\boldsymbol{D}$, referred to as the degree matrix, is a diagonal
matrix with $d_{ii}=\sum_j a_{ij}$, and $\hat{\epsilon}$ is a
small positive value to ensure $\boldsymbol{W}$ to be full rank.


It can be shown that $\boldsymbol{W}$ defined in (\ref{W-1}) and
(\ref{W-2}) promotes low-rankness as well as smoothness of
$\boldsymbol{X}$. To illustrate this, we first introduce the
following lemma.
\newtheorem{lemma}{Lemma}
\begin{lemma}
For a positive-definite matrix
$\boldsymbol{W}\in\mathbb{R}^{M\times M}$, the following equality
holds valid
\begin{align}
\log|\boldsymbol{X}\boldsymbol{X}^T+
    \boldsymbol{W}^{-1}|=\log|\boldsymbol{W}^{-1}|+
    \log|\boldsymbol{I}+\boldsymbol{X}^T\boldsymbol{W}\boldsymbol{X}|
\end{align}
for any $\boldsymbol{X}\in\mathbb{R}^{M\times N}$. \label{lemma1}
\end{lemma}
\begin{proof}
See Appendix \ref{appA}.
\end{proof}

From Lemma \ref{lemma1}, we have
\begin{align}
\log p(\boldsymbol{X})\propto&
-\log|\boldsymbol{X}\boldsymbol{X}^T+\boldsymbol{W}^{-1}|
\nonumber\\
\propto&-\log|\boldsymbol{I}+\boldsymbol{X}^T\boldsymbol{W}\boldsymbol{X}|
\nonumber\\
=&-\sum_{n=1}^{N}\log(\tilde{\lambda}_n+1)
\end{align}
where $\tilde{\lambda}_n$ is the $n$th eigenvalue associated with
$\boldsymbol{X}^T\boldsymbol{W}\boldsymbol{X}$. We see that
maximizing the prior distribution is equivalent to minimizing
$\sum_{n=1}^{N}\log(\tilde{\lambda}_n+1)$ with respect to
$\{\tilde{\lambda}_n\}$. As discussed earlier, this log-sum
functional is a sparsity-promoting functional which encourages a
sparse solution $\{\tilde{\lambda}_n\}$. As a result, the matrix
$\boldsymbol{X}^T\boldsymbol{W}\boldsymbol{X}$ has a low rank.
Since $\boldsymbol{W}$ is full rank, this implies that
$\boldsymbol{X}$ has a low-rank structure. On the other hand,
notice that
$\text{tr}(\boldsymbol{X}^T\boldsymbol{W}\boldsymbol{X})$ is the
first-order approximation of
$\log|\boldsymbol{I}+\boldsymbol{X}^T\boldsymbol{W}\boldsymbol{X}|$.
Therefore minimizing
$\log|\boldsymbol{I}+\boldsymbol{X}^T\boldsymbol{W}\boldsymbol{X}|$
will reduce the value of
$\text{tr}(\boldsymbol{X}^T\boldsymbol{W}\boldsymbol{X})$.
Clearly, for $\boldsymbol{W}$ defined in (\ref{W-1}) and
(\ref{W-2}), a smoother solution results in a smaller value of
$\text{tr}(\boldsymbol{X}^T\boldsymbol{W}\boldsymbol{X})$.
Therefore when $\boldsymbol{W}$ is chosen to be (\ref{W-1}) or
(\ref{W-2}), the resulting prior distribution $p(\boldsymbol{X})$
has the potential to encourage a low-rank and smooth solution.

\emph{Remarks:} Our proposed hierarchical Gaussian prior model can
be considered as a generalization of the prior model in
\cite{XinWang2016}. Notice that in \cite{XinWang2016}, the
precision matrix in the prior model is assumed to be a
deterministic parameter, whereas it is treated as a random
variable and assigned a Wishart prior distribution in our model.
This generalization offers more flexibility in modeling the
underlying latent matrix. As discussed earlier, the parameter
$\boldsymbol{W}$ can be devised to capture additional prior
knowledge about the latent matrix, and such a careful choice of
$\boldsymbol{W}$ can help substantially improve the recovery
performance, as corroborated by our experimental results.




\section{Variational Bayesian Inference} \label{sec:VB}
\subsection{Review of The Variational Bayesian Methodology}
Before proceeding, we firstly provide a brief review of the
variational Bayesian (VB) methodology. In a probabilistic model,
let $\boldsymbol{y}$ and $\boldsymbol{\theta}$ denote the observed
data and the hidden variables, respectively. It is straightforward
to show that the marginal probability of the observed data can be
decomposed into two terms \cite{Bishop07}
\begin{align}
\ln p(\boldsymbol{y})=L(q)+\text{KL}(q|| p),
\label{variational-decomposition}
\end{align}
where
\begin{align}
L(q)=\int q(\boldsymbol{\theta})\ln
\frac{p(\boldsymbol{y},\boldsymbol{\theta})}{q(\boldsymbol{\theta})}d\boldsymbol{\theta}\label{Lq}
\end{align}
and
\begin{align}
\text{KL}(q|| p)=-\int q(\boldsymbol{\theta})\ln
\frac{p(\boldsymbol{\theta}|\boldsymbol{y})}{q(\boldsymbol{\theta})}d\boldsymbol{\theta},
\end{align}
where $q(\boldsymbol{\theta})$ is any probability density
function, $\text{KL}(q|| p)$ is the Kullback-Leibler divergence
\cite{KullbackLeibler51} between
$p(\boldsymbol{\theta}|\boldsymbol{y})$ and
$q(\boldsymbol{\theta})$. Since $\text{KL}(q|| p)\geq 0$, it
follows that $L(q)$ is a rigorous lower bound for $\ln
p(\boldsymbol{y})$. Moreover, notice that the left hand side of
(\ref{variational-decomposition}) is independent of
$q(\boldsymbol{\theta})$. Therefore maximizing $L(q)$ is
equivalent to minimizing $\text{KL}(q|| p)$, and thus the
posterior distribution $p(\boldsymbol{\theta}|\boldsymbol{y})$ can
be approximated by $q(\boldsymbol{\theta})$ through maximizing
$L(q)$.

The significance of the above transformation is that it
circumvents the difficulty of computing the posterior probability
$p(\boldsymbol{\theta}|\boldsymbol{y})$, when it is
computationally intractable. For a suitable choice for the
distribution $q(\boldsymbol{\theta})$, the quantity $L(q)$ may be
more amiable to compute. Specifically, we could assume some
specific parameterized functional form for
$q(\boldsymbol{\theta})$ and then maximize $L(q)$ with respect to
the parameters of the distribution. A particular form of
$q(\boldsymbol{\theta})$ that has been widely used with great
success is the factorized form over the component variables
$\{\theta_i\}$ in $\boldsymbol{\theta}$ \cite{TzikasLikas08}, i.e.
$q(\boldsymbol{\theta})=\prod_i q_i(\theta_i)$. We therefore can
compute the posterior distribution approximation by finding
$q(\boldsymbol{\theta})$ of the factorized form that maximizes the
lower bound $L(q)$. The maximization can be conducted in an
alternating fashion for each latent variable, which leads to
\cite{TzikasLikas08}
\begin{align}
q_i(\theta_i)=\frac{e^{\langle\ln
        p(\boldsymbol{y},\boldsymbol{\theta})\rangle_{k\neq
            i}}}{\int e^{\langle\ln
        p(\boldsymbol{y},\boldsymbol{\theta})\rangle_{k\neq i}}d\theta_i},
\label{general-update}
\end{align}
where $\langle\cdot\rangle_{k\neq i}$ denotes the expectation with
respect to the distributions $q_k(\theta_k)$ for all $k\neq i$. By
taking the logarithm on both sides of (\ref{general-update}), it
can be equivalently written as
\begin{align}
\ln q_i(\theta_i)=\langle\ln
p(\boldsymbol{y},\boldsymbol{\theta})\rangle_{k\neq
i}+\text{constant}. \label{ln-general-update}
\end{align}

\subsection{Proposed Algorithm}
We now proceed to perform variational Bayesian inference for the
proposed hierarchical model. Let
$\boldsymbol{\theta}\triangleq\{\boldsymbol{X},\boldsymbol{\Sigma},
\gamma\}$ denote all hidden variables. Our objective is to find
the posterior distribution
$p(\boldsymbol{\theta}|\boldsymbol{y})$. Since
$p(\boldsymbol{\theta}|\boldsymbol{y})$ is usually computationally
intractable, we, following the idea of \cite{TzikasLikas08},
approximate $p(\boldsymbol{\theta}|\boldsymbol{y})$ as
$q(\boldsymbol{X},\boldsymbol{\Sigma}, \gamma)$ which has a
factorized form over the hidden variables
$\{\boldsymbol{X},\boldsymbol{\Sigma}, \gamma\}$, i.e.
\begin{align}
q(\boldsymbol{X},\boldsymbol{\Sigma}, \gamma) =
q_x(\boldsymbol{X})q_{\Sigma}(\boldsymbol{\Sigma})q_{\gamma}(\gamma).
\end{align}
As mentioned in the previous subsection, the maximization of
$L(q)$ can be conducted in an alternating fashion for each latent
variable, which leads to (details of the derivation can be found
in \cite{TzikasLikas08})
\begin{align}
\ln q_x(\boldsymbol{X})=&\langle\ln p(\boldsymbol{\Sigma},
\gamma)\rangle_{q_{\Sigma}(\boldsymbol{\Sigma})q_{\gamma}(\gamma)} + \text{constant}, \nonumber\\
\ln q_{\Sigma}(\boldsymbol{\Sigma})=&\langle\ln p(\boldsymbol{X},
\gamma)\rangle_{q_x(\boldsymbol{X})q_{\gamma}(\gamma)} + \text{constant}, \nonumber\\
\ln q_{\gamma}(\gamma)=&\langle\ln
p(\boldsymbol{X},\boldsymbol{\Sigma},
)\rangle_{q_x(\boldsymbol{X})q_{\Sigma}(\boldsymbol{\Sigma})} +
\text{constant}, \nonumber
\end{align}
where $\langle\rangle_{q_1(\cdot)\ldots q_K(\cdot)}$ denotes the
expectation with respect to (w.r.t.) the distributions
$\{q_k(\cdot)\}_{k=1}^K$. Details of this Bayesian inference
scheme are provided next.


\textbf{\emph{1).} Update of $q_{x}(\boldsymbol{X})$}: The
calculation of $q_{x}(\boldsymbol{X})$ can be decomposed into a
set of independent tasks, with each task computing the posterior
distribution approximation for each column of $\boldsymbol{X}$,
i.e. $q_{x}(\boldsymbol{x}_n)$. We have
\begin{align}
\ln
q_x(\boldsymbol{x}_n)&\propto\langle\ln[p(\boldsymbol{y}_n|\boldsymbol{x}_n)
p(\boldsymbol{x}_n|\boldsymbol{\Sigma})]\rangle_{q_{\Sigma}(\boldsymbol{\Sigma})q_{\gamma}(\gamma)}\nonumber\\
&\propto\langle-\gamma(\boldsymbol{y}_n-\boldsymbol{x}_n)^T
\boldsymbol{O}_n(\boldsymbol{y}_n-\boldsymbol{x}_n)-\boldsymbol{x}_n^T
\boldsymbol{\Sigma}\boldsymbol{x}_n\rangle\nonumber\\
&\propto-\boldsymbol{x}_n^T(\langle\gamma\rangle\boldsymbol{O}_n+\langle\boldsymbol{\Sigma}\rangle)
\boldsymbol{x}_n+2\langle\gamma\rangle\boldsymbol{x}_n^T\boldsymbol{O}_n\boldsymbol{y}_n\label{x-likelihood}
\end{align}
where $\boldsymbol{y}_n$ denotes the $n$th column of
$\boldsymbol{Y}$ and
$\boldsymbol{O}_n\triangleq\text{diag}(\boldsymbol{o}_n)$, with
$\boldsymbol{o}_n$ being the $n$th column of
$\boldsymbol{\Omega}$. From (\ref{x-likelihood}), it can be seen
that $\boldsymbol{x}_n$ follows a Gaussian distribution
\begin{align}
q_x(\boldsymbol{x}_n)=\mathcal{N}(\boldsymbol{x}_n|\boldsymbol{\mu}_n,\boldsymbol{Q}_n)\label{x-update1}
\end{align}
with $\boldsymbol{\mu}_n$ and $\boldsymbol{Q}_n$ given as
\begin{align}
\boldsymbol{\mu}_n&=\langle\gamma\rangle\boldsymbol{Q}_n\boldsymbol{O}_n\boldsymbol{y}_n \label{mu-update}\\
\boldsymbol{Q}_n&=(\langle\gamma\rangle\boldsymbol{O}_n+\langle\boldsymbol{\Sigma}\rangle)^{-1}\label{Q-update}
\end{align}
We see that to calculate $q_{x}(\boldsymbol{x}_n)$, we need to
perform an inverse operation of an $M\times M$ matrix which
involves a computational complexity of $\mathcal{O}(M^3)$.

\textbf{\emph{2).} Update of $q_{\Sigma}(\boldsymbol{\Sigma})$}:
The approximate posterior $q_{\Sigma}(\boldsymbol{\Sigma})$ can be
obtained as
\begin{align}
&\ln q_{\Sigma}(\boldsymbol{\Sigma})\nonumber\\
\propto&
\langle\ln[\prod\limits_{n=1}^{N}p(\boldsymbol{x}_n|\boldsymbol{\Sigma})
p(\boldsymbol{\Sigma})]\rangle_{q_x(\boldsymbol{X})}\nonumber\\
\propto&\langle\frac{N}{2}\ln|\boldsymbol{\Sigma}|-\frac{1}{2}\text{tr}(\boldsymbol{X}^T\boldsymbol{\Sigma}\boldsymbol{X})
+\frac{\nu-M-1}{2}\ln|\boldsymbol{\Sigma}|
\nonumber\\&-\frac{1}{2}\text{tr}(\boldsymbol{W}^{-1}\boldsymbol{\Sigma})\rangle\nonumber\\
\propto&\frac{\nu+N-M-1}{2}\ln|\boldsymbol{\Sigma}|-\frac{1}{2}
\text{tr}((\boldsymbol{W}^{-1}+\langle\boldsymbol{X}\boldsymbol{X}^T\rangle)\boldsymbol{\Sigma})\label{Sigma-likelihood}
\end{align}
From (\ref{Sigma-likelihood}), it can be seen that
$\boldsymbol{\Sigma}$ follows a Wishart distribution, i.e.
\begin{align}
q_{\Sigma}(\boldsymbol{\Sigma})=\text{Wishart}(\boldsymbol{\Sigma};\hat{\boldsymbol{W}},\hat{\nu})\label{Sigma-update1}
\end{align}
where
\begin{align}
\hat{\boldsymbol{W}}&=(\boldsymbol{W}^{-1}+\langle\boldsymbol{X}\boldsymbol{X}^T\rangle)^{-1}\\
\hat{\nu}&=\nu+N
\end{align}

\textbf{\emph{3).} Update of $q_{\gamma}(\boldsymbol{\gamma})$}:
The variational optimization of $q_{\gamma}(\gamma)$ yields
\begin{align}
\ln q_{\gamma}(\gamma)\propto&\langle\ln
p(\boldsymbol{Y}|\boldsymbol{X},\gamma)p(\gamma)\rangle_{q_x(\boldsymbol{X})}
\nonumber\\
\propto&\langle\ln\prod_{(m,n)\in\mathbb{S}}p(y_{mn}|x_{mn},\gamma)p(\gamma)\rangle
\nonumber\\
\propto&
\langle\frac{L}{2}\ln\gamma-\frac{\gamma}{2}\sum_{(m,n)\in\mathbb{S}}(y_{mn}-x_{mn})^2
+(c-1)\ln\gamma-d\gamma\rangle \nonumber\\
=&\bigg(\frac{L}{2}+c-1\bigg)\ln\gamma-
\bigg(\frac{1}{2}\sum_{(m,n)\in\mathbb{S}}\langle(y_{mn}-x_{mn})^2\rangle+d\bigg)\gamma
\end{align}
where $x_{mn}$ and $y_{mn}$ denote the $(m,n)$th entry of
$\boldsymbol{X}$ and $\boldsymbol{Y}$, respectively,
$\mathbb{S}\triangleq\{(m,n)|\Omega_{mn}=1\}$ is an index set
consisting of indices of those observed entries, and $L\triangleq
|\mathbb{S}|$ is the cardinality of the set $\mathbb{S}$, in which
$\Omega_{mn}$ denotes the $(m,n)$th entry of
$\boldsymbol{\Omega}$.

It is easy to verify that $q_{\gamma}(\gamma)$ follows a Gamma
distribution
\begin{align}
q_{\gamma}(\gamma)=\text{Gamma}(\gamma|\tilde{c},\tilde{d})
\label{gamma-update}
\end{align}
with the parameters $\tilde{c}$ and $\tilde{d}$ given respectively
by
\begin{align}
\tilde{c}=&\frac{L}{2}+c, \nonumber\\
\tilde{d}=&\frac{1}{2}\sum_{(m,n)\in\mathbb{S}}\langle(y_{mn}-x_{mn})^2\rangle+d
\end{align}
where
\begin{align}
\langle(y_{mn}-x_{mn})^2\rangle=y_{mn}^2-2y_{mn}\langle
x_{mn}\rangle+\langle x_{mn}^2\rangle
\end{align}

Some of the expectations and moments used during the update are
summarized as
\begin{align}
\langle\boldsymbol{\Sigma}\rangle&=\hat{\boldsymbol{W}}\hat{\nu}\\
\langle\boldsymbol{X}\boldsymbol{X}^T\rangle&=\langle\boldsymbol{X}\rangle
\langle\boldsymbol{X}\rangle^T+\sum\limits_{n=1}^{N}\boldsymbol{Q}_n\label{XX}\\
\langle x_{mn}^2\rangle&=\langle x_{mn}\rangle^2+Q_n(m,m)
\end{align}
where $Q_n(m,m)$ denotes the $m$th diagonal entry of
$\boldsymbol{Q}_n$.

For clarity, we summarize our algorithm as follows.

\begin{algorithm}
    \renewcommand{\algorithmicrequire}{\textbf{Input:}}
    \renewcommand\algorithmicensure {\textbf{Output:}}
    \caption{VB Algorithm for Matrix Completion}
    \begin{algorithmic}[0]
        \REQUIRE $\boldsymbol{Y}$, $\boldsymbol{\Omega}$, $\nu$ and $\boldsymbol{W}$.
        \ENSURE $q_{x}(\boldsymbol{X})$, $q_{\Sigma}(\boldsymbol{\Sigma})$, $q_{\gamma}(\gamma)$.
        \STATE Initialize $\langle\boldsymbol{\Sigma}\rangle$ and
        $\langle\gamma\rangle$;
        \WHILE {not converge}
        \FOR{$n=1$ to $N$}
        \STATE Update $q_x({\boldsymbol{x}_n})$ via
        (\ref{x-update1}),
        with $q_{\Sigma}(\boldsymbol{\Sigma})$ and $q_{\gamma}(\gamma)$ fixed;
        \ENDFOR
        \STATE Update $q_{\Sigma}({\boldsymbol{\Sigma}})$ via
        (\ref{Sigma-update1}),
        with $q_{x}(\boldsymbol{X})$ and $q_{\gamma}(\gamma)$ fixed;
        \STATE Update $q_{\gamma}(\gamma)$ via (\ref{gamma-update});
        \ENDWHILE
    \end{algorithmic}
    \label{algorithm-vb}
\end{algorithm}

It can be easily checked that the computational complexity of our
proposed method is dominated by the update of the posterior
distribution $q_x(\boldsymbol{X})$, which requires computing an
$M\times M$ matrix inverse $N$ times and therefore has a
computational complexity scaling as $\mathcal{O}(M^3 N)$. This
makes the application of our proposed method to large data sets
impractical. To address this issue, in the following, we develop a
computationally efficient algorithm which obtains an approximation
of $q_x(\boldsymbol{X})$ by resorting to the generalized
approximate message passing (GAMP) technique \cite{Rangan11}.




\section{VB-GAMP}
\label{sec:GAMP-VB} GAMP is a low-complexity Bayesian iterative
technique recently developed in \cite{DonohoMaleki10,Rangan11} for
obtaining approximate marginal posteriors. Note that the GAMP
algorithm requires that both the prior distribution and the noise
distribution have factorized forms \cite{Rangan11}. Nevertheless,
in our model, the prior distribution
$p(\boldsymbol{x}_n|\boldsymbol{\Sigma})$ has a non-factorizable
form, in which case the GAMP technique cannot be directly applied.
To address this issue, we first construct a surrogate problem
which aims to recover $\boldsymbol{x}\in\mathbb{R}^{M}$ from
linear measurements $\boldsymbol{b}\in\mathbb{R}^{M}$:
\begin{align}
\boldsymbol{b}=\boldsymbol{U}^T\boldsymbol{x}+\boldsymbol{e}
\label{surrogate-problem}
\end{align}
where $\boldsymbol{U}\in\mathbb{C}^{M\times M}$ is obtained by
performing a singular value decomposition of
$\langle\boldsymbol{\Sigma}\rangle=\boldsymbol{U}\boldsymbol{S}\boldsymbol{U}^T$,
$\boldsymbol{U}$ is a unitary matrix and $\boldsymbol{S}$ is a
diagonal matrix with its diagonal elements equal to the singular
values of $\langle\boldsymbol{\Sigma}\rangle$, and
$\boldsymbol{e}$ denotes the additive Gaussian noise with zero
mean and covariance matrix $\boldsymbol{S}^{-1}$. We assume that
entries of $\boldsymbol{x}$ are mutually independent and follow
the following distribution:
\begin{align}
p(x_{m})=
\begin{cases}
\mathcal{N}(\kappa_{m},\xi^{-1}) & \text{if $\pi_{m}=1$}\\
C, & \text{if $\pi_{m}=0$}
\end{cases} \label{sp-x-prior}
\end{align}
where $\pi_{m}$, $x_{m}$, and $\kappa_{m}$ denote the $m$th entry
of $\boldsymbol{\pi}$, $\boldsymbol{x}$, and
$\boldsymbol{\kappa}$, respectively, $C$ is a constant,
$\boldsymbol{\pi}$, $\boldsymbol{\kappa}\in\mathbb{R}^{M\times 1}$
and $\xi$ are known parameters. It is noted that although
$p(x_{m})=C$ is an improper prior distribution, it can often be
used provided the corresponding posterior distribution can be
correctly normalized \cite{Bishop07}. Considering the surrogate
problem (\ref{surrogate-problem}), the posterior distribution of
$\boldsymbol{x}$ can be calculated as
\begin{align}
p(\boldsymbol{x}|\boldsymbol{b})&\propto
p(\boldsymbol{b}|\boldsymbol{x})
p(\boldsymbol{x})\nonumber\\
&\propto p(\boldsymbol{b}|\boldsymbol{x})\prod_{m\in S}p(x_{m})\nonumber\\
&=\mathcal{N}(\boldsymbol{U}^T\boldsymbol{x},\boldsymbol{S}^{-1})\prod_{m\in
S}\mathcal{N}(\kappa_{m},\xi^{-1})
\end{align}
where $S\triangleq\{m| \pi_m=1\}$.


Taking the logarithm of $p(\boldsymbol{x}|\boldsymbol{b})$, we
have
\begin{align}
\ln p(\boldsymbol{x}|\boldsymbol{b})&\propto
-\frac{1}{2}(\boldsymbol{b}-\boldsymbol{U}^T\boldsymbol{x})^T\boldsymbol{S}
(\boldsymbol{b}-\boldsymbol{U}^T\boldsymbol{x})\nonumber\\
&\quad-\frac{1}{2}\xi\sum_{m\in S}(x_{m}-\kappa_{m})^2\nonumber\\
&=-\frac{1}{2}(\boldsymbol{b}-\boldsymbol{U}^T\boldsymbol{x})^T\boldsymbol{S}
(\boldsymbol{b}-\boldsymbol{U}^T\boldsymbol{x})\nonumber\\
&\quad-\frac{1}{2}\xi(\boldsymbol{x}-\boldsymbol{\kappa})^T\boldsymbol{\Pi}(\boldsymbol{x}-\boldsymbol{\kappa})\nonumber\\
&\propto-\frac{1}{2}\boldsymbol{x}^T(\boldsymbol{U}\boldsymbol{S}\boldsymbol{U}^T+
\xi\boldsymbol{\Pi})\boldsymbol{x}^T+(\boldsymbol{b}^T\boldsymbol{S}\boldsymbol{U}^T+
\xi\boldsymbol{\kappa}^T\boldsymbol{\Pi}) \boldsymbol{x}
\end{align}
where $\boldsymbol{\Pi}$ is a diagonal matrix with its $m$th
diagonal entry equal to $\pi_m$. Clearly,
$p(\boldsymbol{x}|\boldsymbol{b})$ follows a Gaussian distribution
with its mean $\boldsymbol{\mu}$ and covariance matrix
$\boldsymbol{Q}$ given by
\begin{align}
\boldsymbol{\mu}=&\boldsymbol{Q}(\boldsymbol{U}\boldsymbol{S}\boldsymbol{b}+
\xi\boldsymbol{\Pi}\boldsymbol{\kappa}) \label{eqn1}\\
\boldsymbol{Q}=&(\boldsymbol{U}\boldsymbol{S}\boldsymbol{U}^T+\xi\boldsymbol{\Pi})^{-1}
=(\langle\boldsymbol{\Sigma}\rangle+\xi\boldsymbol{\Pi})^{-1}
\label{eqn2}
\end{align}
Comparing (\ref{mu-update})--(\ref{Q-update}) with
(\ref{eqn1})--(\ref{eqn2}), we can readily verify that when
$\boldsymbol{b}=\boldsymbol{0}$,
$\boldsymbol{\kappa}=\boldsymbol{y}_n$,
$\boldsymbol{\pi}=\boldsymbol{o}_n$ (i.e.
$\boldsymbol{\Pi}=\boldsymbol{O}_n$),
and $\xi=\langle\gamma\rangle$,
$p(\boldsymbol{x}|\boldsymbol{b})$ is exactly the desired
posterior distribution $q_x(\boldsymbol{x}_n)$. Meanwhile, notice
that for the surrogate problem (\ref{surrogate-problem}), both the
prior distribution and the noise distribution are factorizable.
Hence the GAMP algorithm can be directly applied to
(\ref{surrogate-problem}) to find an approximation of the
posterior distribution $p(\boldsymbol{x}|\boldsymbol{b})$. By
setting $\boldsymbol{b}=\boldsymbol{0}$,
$\boldsymbol{\kappa}=\boldsymbol{y}_n$,
$\boldsymbol{\pi}=\boldsymbol{o}_n$, $\xi=\langle\gamma\rangle$, an approximate of
$q_x(\boldsymbol{x}_n)$ in (\ref{x-update1}) can be efficiently
obtained. We now proceed to derive the GAMP algorithm for the
surrogate problem (\ref{surrogate-problem}).

\subsection{Solving (\ref{surrogate-problem}) via GAMP}
GAMP was developed in a message passing-based framework. By using
central-limit-theorem approximations, message passing between
variable nodes and factor nodes can be greatly simplified, and the
loopy belief propagation on the underlying factor graph can be
efficiently performed. As noted in \cite{Rangan11}, the
central-limit-theorem approximations become exact in the
large-system limit under an i.i.d. zero-mean sub-Gaussian
measurement matrix.

Firstly, GAMP approximates the true marginal posterior
distribution $p(x_{m}|\boldsymbol{b})$ by
\begin{align}
\hat{p}(x_{m}|\boldsymbol{b},\hat{r}_{m},\tau_{m}^r)
=&\frac{p(x_{m})\mathcal{N}(x_{m}|\hat{r}_{m},\tau_{m}^r)}{\int_x
    p(x_{m})\mathcal{N}(x_{m}|\hat{r}_{m},\tau_{m}^r)}
\label{eqn-1}
\end{align}
where $\hat{r}_{m}$ and $\tau_{m}^r$ are quantities iteratively
updated during the iterative process of the GAMP algorithm. Here,
we have dropped their explicit dependence on the iteration number
$k$ for simplicity. For the case $\pi_{m}=1$, substituting the
prior distribution (\ref{sp-x-prior}) into (\ref{eqn-1}), it can
be easily verified that the approximate posterior
$\hat{p}(x_{m}|\boldsymbol{b},\hat{r}_{m},\tau_{m}^r)$ follows a
Gaussian distribution with its mean and variance given
respectively as
\begin{align}
\mu_{m}^x&=\phi_{m}^x(\xi \kappa_{m}+\hat{r}_{m}/\tau_{m}^r) \label{x-post-mean-1} \\
\phi_{m}^x&=\frac{\tau_{m}^r}{1+\xi\tau_{m}^r}
\label{x-post-var-1}
\end{align}
Similarly, for the case $\pi_{m}=0$, substituting the prior
distribution (\ref{sp-x-prior}) into (\ref{eqn-1}), the
approximate posterior
$\hat{p}(x_{m}|\boldsymbol{b},\hat{r}_{m},\tau_{m}^r)$ follows a
Gaussian distribution with its mean and variance given
respectively as
\begin{align}
\mu_{m}^x&=\hat{r}_{m} \label{x-post-mean-0} \\
\phi_{m}^x&=\tau_{m}^r\label{x-post-var-0}
\end{align}

Another approximation is made to the noiseless output
$z_{i}\triangleq\boldsymbol{u}_i^T\boldsymbol{x}$, where
$\boldsymbol{u}_i^T$ denotes the $i$th row of $\boldsymbol{U}^T$.
GAMP approximates the true marginal posterior
$p(z_{i}|\boldsymbol{b})$ by
\begin{align}
\hat{p}(z_{i}|\boldsymbol{b},\hat{p}_{i},\tau_{i}^p)=\frac{p(b_i|z_{i})
    \mathcal{N}(z_{i}|\hat{p}_{i},\tau_{i}^p)}{\int_z
    p(b_{i}|z_{i})
    \mathcal{N}(z_{i}|\hat{p}_{i},\tau_{i}^p)}
\end{align}
where $\hat{p}_{i}$ and $\tau_{i}^p$ are quantities iteratively
updated during the iterative process of the GAMP algorithm. Again,
here we dropped their explicit dependence on the iteration number
$k$. Under the additive white Gaussian noise assumption, we have
$p(b_i|z_{i})=\mathcal{N}(b_{i}|z_{i},s_{i}^{-1})$, where $s_{i}$
denotes the $i$th diagonal element of $\boldsymbol{S}$. Thus
$\hat{p}(z_{i}|\boldsymbol{b},\hat{p}_{i},\tau_{i}^p)$ also
follows a Gaussian distribution with its mean and variance given
by
\begin{align}
\mu_{i}^z=&\frac{\tau_{i}^p s_{i}
    b_{i}+\hat{p}_{i}}{1+s_{i}\tau_{i}^p} \label{z-post-mean}
\\
\phi_{i}^z=&\frac{\tau_{i}^p}{1+s_{i}\tau_{i}^p}
\label{z-post-var}
\end{align}

With the above approximations, we can now define the following two
scalar functions: $g_{\text{in}}(\cdot)$ and
$g_{\text{out}}(\cdot)$ that are used in the GAMP algorithm. The
input scalar function $g_{\text{in}}(\cdot)$ is simply defined as
the posterior mean $\mu_m^x$, i.e.
\begin{align}
g_{\text{in}}(\hat{r}_{m},\tau_{m}^r) =\mu_{m}^x=
\begin{cases}
\phi_{m}^x(\xi \kappa_{m}+\hat{r}_{m}/\tau_{m}^r)&\text{if $\pi_{m}=1$}\\
\hat{r}_{m}&\text{if $\pi_{m}=0$}
\end{cases}
\end{align}
The scaled partial derivative of $\tau_{m}^r
g_{\text{in}}(\hat{r}_{m},\tau_{m}^r)$ with respect to
$\hat{r}_{m}$ is the posterior variance $\phi_{m}^x$, i.e.
\begin{align}
\tau_{m}^r\frac{\partial}{\partial\hat{r}_{m}}g_{\text{in}}(\hat{r}_{m},\tau_{m}^r)
=\phi_{m}^x=\begin{cases}
\frac{\tau_{m}^r}{1+\xi\tau_{m}^r}&\text{if $\pi_{m}=1$}\\
\tau_{m}^r&\text{if $\pi_{m}=0$}
\end{cases}
\end{align}
The output scalar function $g_{\text{out}}(\cdot)$ is related to
the posterior mean $\mu_{i}^z$ as follows
\begin{align}
g_{\text{out}}(\hat{p}_{i},\tau_{i,n}^p)&=\frac{1}{\tau_{i}^p}(\mu_{i}^z-\hat{p}_{i})
=\frac{s_{i}(b_i-\hat{p}_{i})}{1+s_{i}\tau_{i}^p}
\end{align}
The partial derivative of $g_{\text{out}}(\hat{p}_{i},\tau_{i}^p)$
is related to the posterior variance $\phi_{i,n}^z$ in the
following way
\begin{align}
\frac{\partial}{\partial\hat{p}_{i}}g_{\text{out}}(\hat{p}_{i},\tau_{i}^p)=
\frac{\phi_{i}^z-\tau_{i}^p}{(\tau_{i}^p)^2}=\frac{-s_{i}}{(1+s_{i}\tau_{i}^p)}
\end{align}
Given the above definitions of $g_{\text{in}}(\cdot)$ and
$g_{\text{out}}(\cdot)$, the GAMP algorithm tailored to the
considered problem (\ref{surrogate-problem}) can now be summarized
as follows (details of the derivation of the GAMP algorithm can be
found in \cite{Rangan11}), in which $u_{i,m}$ denotes the
$(i,m)$th entry of $\boldsymbol{U}^T$.

\begin{algorithm}
    \renewcommand{\algorithmicrequire}{\textbf{Input:}}
    \renewcommand\algorithmicensure {\textbf{Output:}}
    \caption{GAMP Algorithm}
    \begin{algorithmic}[0]
        \REQUIRE $\boldsymbol{\kappa}$, $\boldsymbol{\pi}$, $\boldsymbol{b}$, and $\xi$.
        \ENSURE $\{\hat{r}_{m},\tau_{m}^r\}$, $\{\hat{p}_{i},\tau_{i}^p\}$, and $\{\mu_{m}^x,\phi_{m}^{x}\}$.
        \STATE Initialization: Set
        $\hat{\psi}_{i}=0,\forall i\in\{1,\ldots,M\}$;
        $\{\mu_{m}^x\}_{m=1}^M$ are
        initialized as the mean variance of the prior distribution,
        and $\{\phi_{m}^{x}\}_{m=1}^M$ are set to small values, say $10^{-5}$.
        \WHILE {not converge}
        \STATE Step 1. $\forall i\in\{1,\ldots,M\}$: \\
        $\begin{aligned}
        \qquad\qquad\quad \hat{z}_{i}=&\sum_m u_{i,m}\mu_{m}^x\\
        \tau_{i}^p=&\sum_m u_{i,m}^2\phi_{m}^{x} \\
        \hat{p}_{i}=&\hat{z}_{i}-\tau_{i}^p\hat{\psi}_{i}\\
        \end{aligned}$ \\
        \STATE Step 2. $\forall i\in\{1,\ldots,M\}$: \\
        $\begin{aligned} \qquad\qquad\quad
        \hat{\psi}_{i}=&g_{\text{out}}(\hat{p}_{i},\tau^p_{i})\\
        \tau^s_{i}=&-\frac{\partial}{\partial\hat{p}_{i}}g_{\text{out}}(\hat{p}_{i},\tau^p_{i})\\
        \end{aligned}$ \\
        \STATE Step 3. $\forall m\in\{1,\ldots,M\}$: \\
        $\begin{aligned} \qquad\qquad\quad
        \tau_{m}^r=&\left(\sum_{i}u_{i,m}^2\tau^{s}_{i}\right)^{-1} \\
        \hat{r}_{m}=&\mu_{m}^x+\tau_{m}^r\sum_{i}u_{i,m}\hat{\psi}_{i} \\
        \end{aligned}$ \\
        \STATE Step 4. $\forall m\in\{1,\ldots,M\}$: \\
        $\begin{aligned} \qquad\qquad\quad
        \mu_{m}^x=&g_{\text{in}}(\hat{r}_{m},\tau_{m}^r)
        \\
        \phi_{m}^{x}=&\tau_{m}^r\frac{\partial}{\partial\hat{r}_{m}}g_{\text{in}}(\hat{r}_{m},\tau_{m}^r)
        \\
        \end{aligned}$ \\
        \ENDWHILE
    \end{algorithmic}
    \label{algorithm-gamp}
\end{algorithm}

\subsection{Discussions}
We have now derived an efficient algorithm to obtain an
approximate posterior distribution of $\boldsymbol{x}$ for
(\ref{surrogate-problem}). Specifically, the true marginal
posterior distribution of $x_m$ is approximated by a Gaussian
distribution
$\hat{p}(x_{m}|\boldsymbol{b},\hat{r}_{m},\tau_{m}^r)$ with its
mean and variance given by
(\ref{x-post-mean-1})--(\ref{x-post-var-1}) or
(\ref{x-post-mean-0})--(\ref{x-post-var-0}), depending on the
value of $\pi_m$. The joint posterior distribution
$p(\boldsymbol{x}|\boldsymbol{b})$ can be approximated as a
product of approximate marginal posterior distributions:
\begin{align}
p(\boldsymbol{x}|\boldsymbol{b})\approx
\hat{p}(\boldsymbol{x}|\boldsymbol{b})
=\prod\limits_{m=1}^{M}\hat{p}(x_{m}|\boldsymbol{b},\hat{r}_{m},\tau_{m}^r)
\end{align}
As indicated earlier, by setting $\boldsymbol{b}=\boldsymbol{0}$,
$\boldsymbol{\kappa}=\boldsymbol{y}_n$,
$\boldsymbol{\pi}=\boldsymbol{o}_n$, and
$\xi=\langle\gamma\rangle$, the posterior distribution
$\hat{p}(\boldsymbol{x}|\boldsymbol{b})$ obtained via the GAMP
algorithm can be used to approximate $q_x(\boldsymbol{x}_n)$ in
(\ref{x-update1}).

We see that to approximate $q_x(\boldsymbol{X})$ by using the
GAMP, we first need to perform a singular value decomposition
(SVD) of $\langle\boldsymbol{\Sigma}\rangle$, which has a
computational complexity of $\mathcal{O}(M^3)$. The GAMP algorithm
used to approximate $q_x(\boldsymbol{x}_n)$ involves very simple
matrix-vector multiplications which has a computational complexity
scaling as $\mathcal{O}(M^2)$. Therefore the overall computational
complexity for updating $q_x(\boldsymbol{X})$ is of order
$\mathcal{O}(M^2 N+M^3)$. In contrast, using
(\ref{mu-update})--(\ref{Q-update}) to update
$q_x(\boldsymbol{X})$ requires a computational complexity of
$\mathcal{O}(N M^3)$. Thus the GAMP technique can help achieve a
significant reduction in the computational complexity as compared
with a direct calculation of $q_x(\boldsymbol{X})$.

For clarity, the VB-GAMP algorithm for matrix completion is
summarized as

\begin{algorithm}
    \renewcommand{\algorithmicrequire}{\textbf{Input:}}
    \renewcommand\algorithmicensure {\textbf{Output:}}
    \caption{VB-GAMP Algorithm for Matrix Completion}
    \begin{algorithmic}[1]
        \REQUIRE $\boldsymbol{Y}$, $\boldsymbol{\Omega}$, $\nu$ and $\boldsymbol{W}$.
        \ENSURE $q_{x}(\boldsymbol{X})$, $q_{\Sigma}(\boldsymbol{\Sigma})$, and $q_{\gamma}(\gamma)$.
        \STATE Initialize $\langle\boldsymbol{X}\rangle$,
        $\langle\boldsymbol{\Sigma}\rangle$;
        \WHILE {not converge}
        \STATE Calculate singular value decomposition of $\langle\boldsymbol{\Sigma}\rangle$;
        \FOR{$n=1$ to $N$}
        \STATE Obtain an approximation of $q_x(\boldsymbol{x}_n)$ via Algorithm 2;
        \ENDFOR
        \STATE Update $q_{\Sigma}({\boldsymbol{\Sigma}})$ via (\ref{Sigma-update1});
        \STATE Update $q_{\gamma}(\gamma)$ via (\ref{gamma-update});
        \ENDWHILE
    \end{algorithmic}
    \label{algorithm-overall}
\end{algorithm}

The proposed method proceeds in a double-loop manner, the outer
loop calculate the variational posterior distributions
$q_{\gamma}(\gamma)$ and $q_{\Sigma}(\boldsymbol{\Sigma})$, and
the inner loop computes an approximation of $q_x(\boldsymbol{X})$.
It is noted that there is no need to wait until the GAMP
converges. Experimental results show that GAMP provides a reliable
approximation of $q_x(\boldsymbol{x}_n)$ even if only a few
iterations are performed. In our experiments, only one iteration
is used to implement GAMP.


\section{Experiments} \label{sec:simulation-results}
In this section, we carry out experiments to illustrate the
performance of our proposed GAMP-assisted Bayesian matrix
completion method with hierarchical Gaussian priors (referred to
as BMC-GP-GAMP). Throughout our experiments, the parameters used
in our method are set to be $a=b=10^{-10}$ and $\nu=1$. Here we
choose a small $\nu$ in order to encourage a low-rank precision
matrix. We compare our method with several state-of-the-art
methods, namely, the variational sparse Bayesian learning method
(also referred to as VSBL) \cite{BabacanLuessi2012} which models
the low-rankness of the matrix as the structural sparsity of its
two factor matrices, the bilinear GAMP-based matrix completion
method (also referred to as BiGAMP-MC) \cite{ParkerSchniter2014b}
which implements the VSBL using bilinear GAMP, the inexact version
of the augmented Lagrange multiplier based matrix completion
method (also referred to ALM-MC) \cite{LinChen2010}, and a
low-rank matrix fitting method (also referred to as LMaFit)
\cite{WenYin12} which iteratively minimizes the fitting error and
estimates the rank of the matrix. It should be noted that both
VSBL and LMaFit require to set an over-estimated rank. Note that
we did not include \cite{XinWang2016} in our experiments due to
its prohibitive computational complexity when the matrix dimension
is large. Codes of our proposed algorithm along with other
competing algorithms are available at
\url{http://www.junfang-uestc.net/codes/LRMC.rar}, in which codes
of other competing algorithms are obtained from their respective
websites.



\begin{figure*}
    \centering
    \subfloat[]{
        \includegraphics [width=200pt]{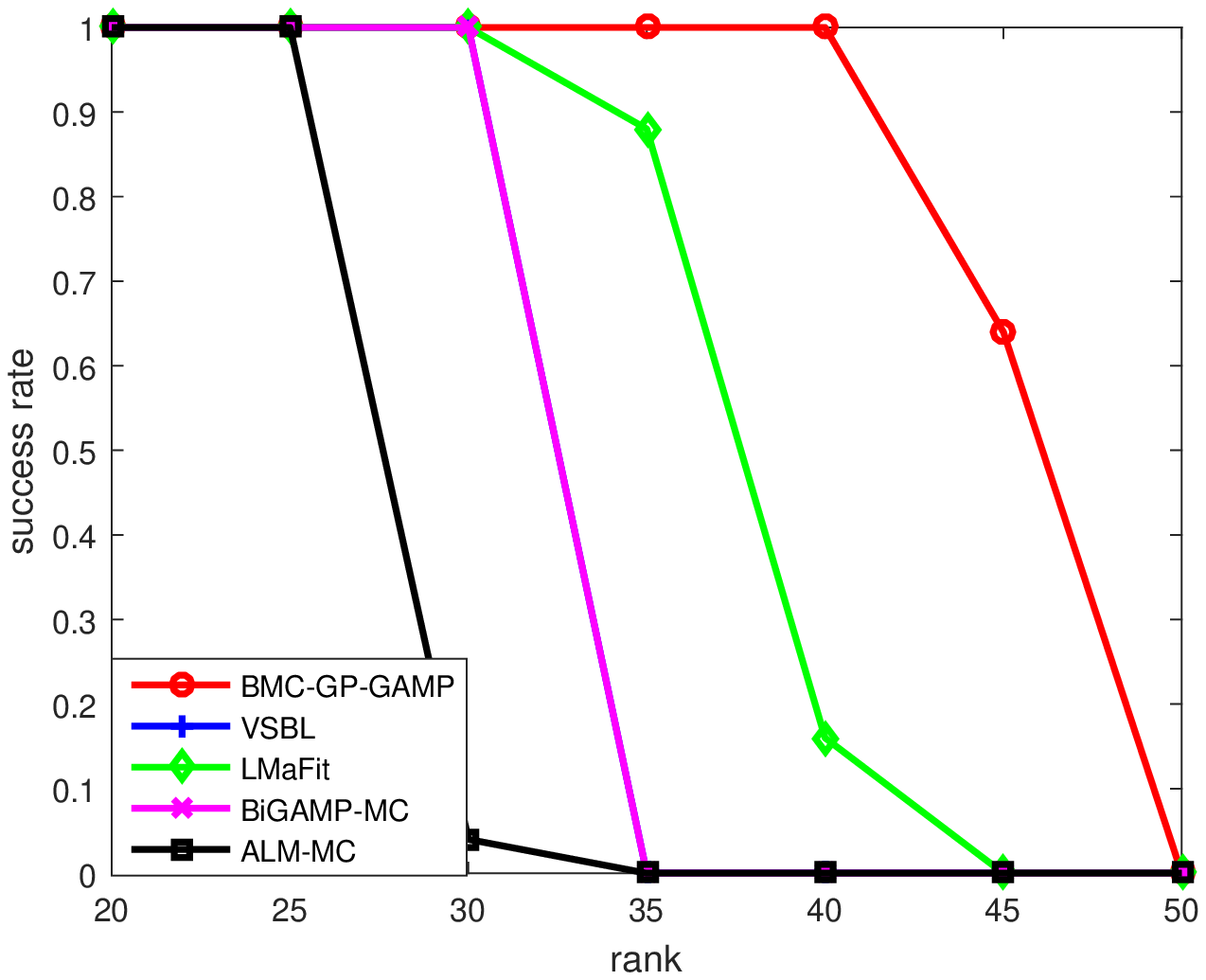}
    }
    \hspace{20pt}
    \subfloat[]{
        \includegraphics [width=200pt]{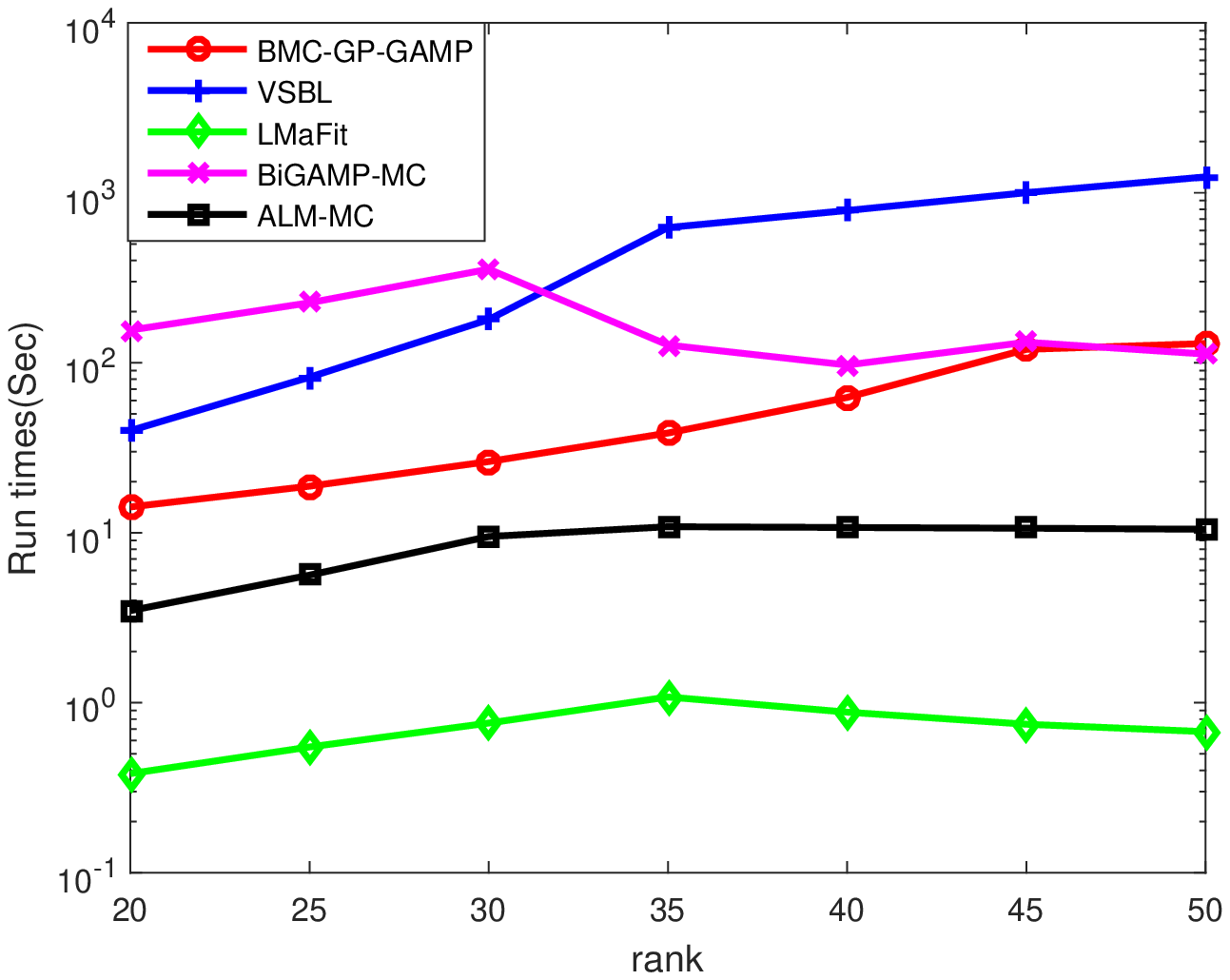}
    }
    \\
    \subfloat[]{
        \includegraphics [width=200pt]{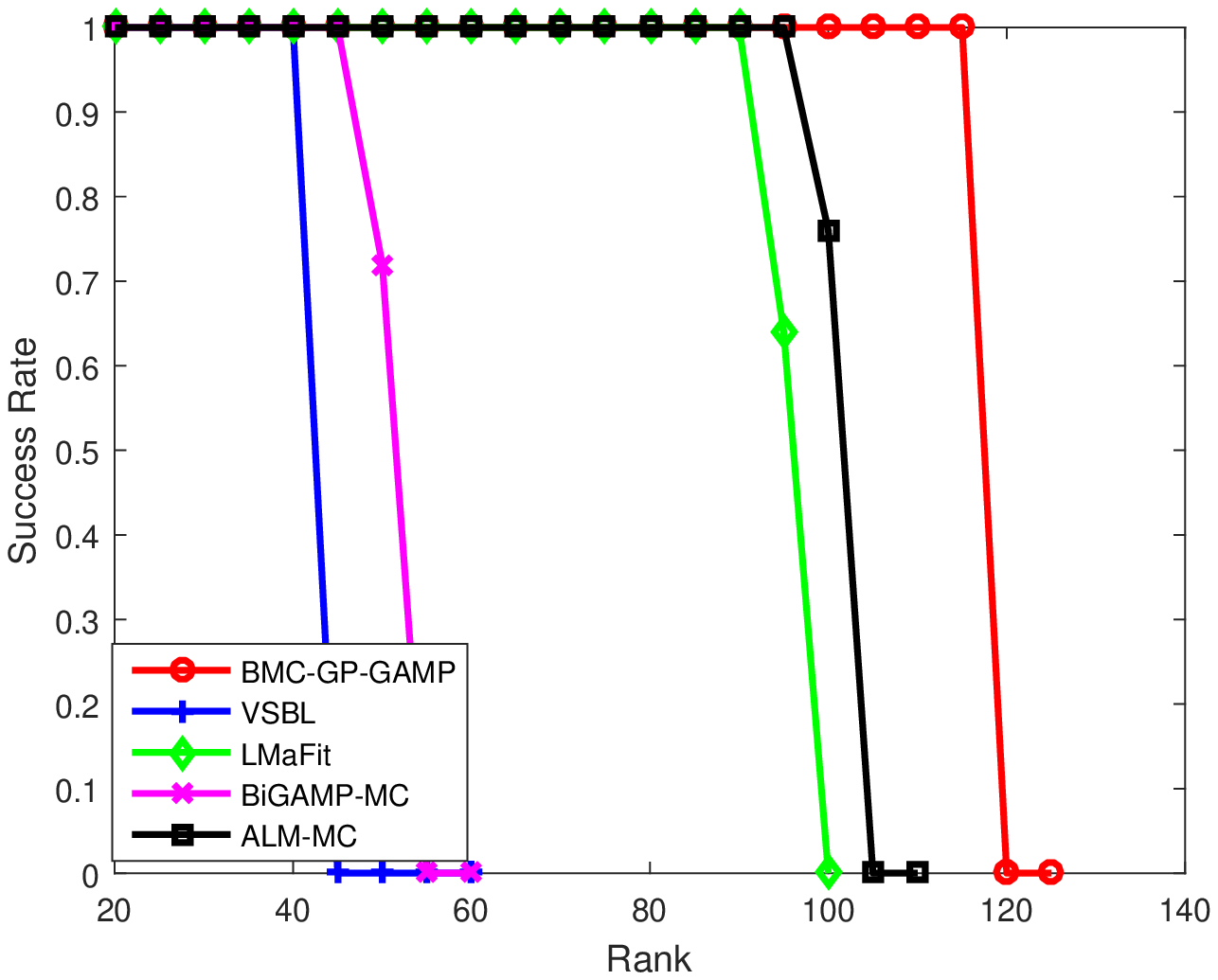}
    }
    \hspace{20pt}
    \subfloat[]{
        \includegraphics [width=200pt]{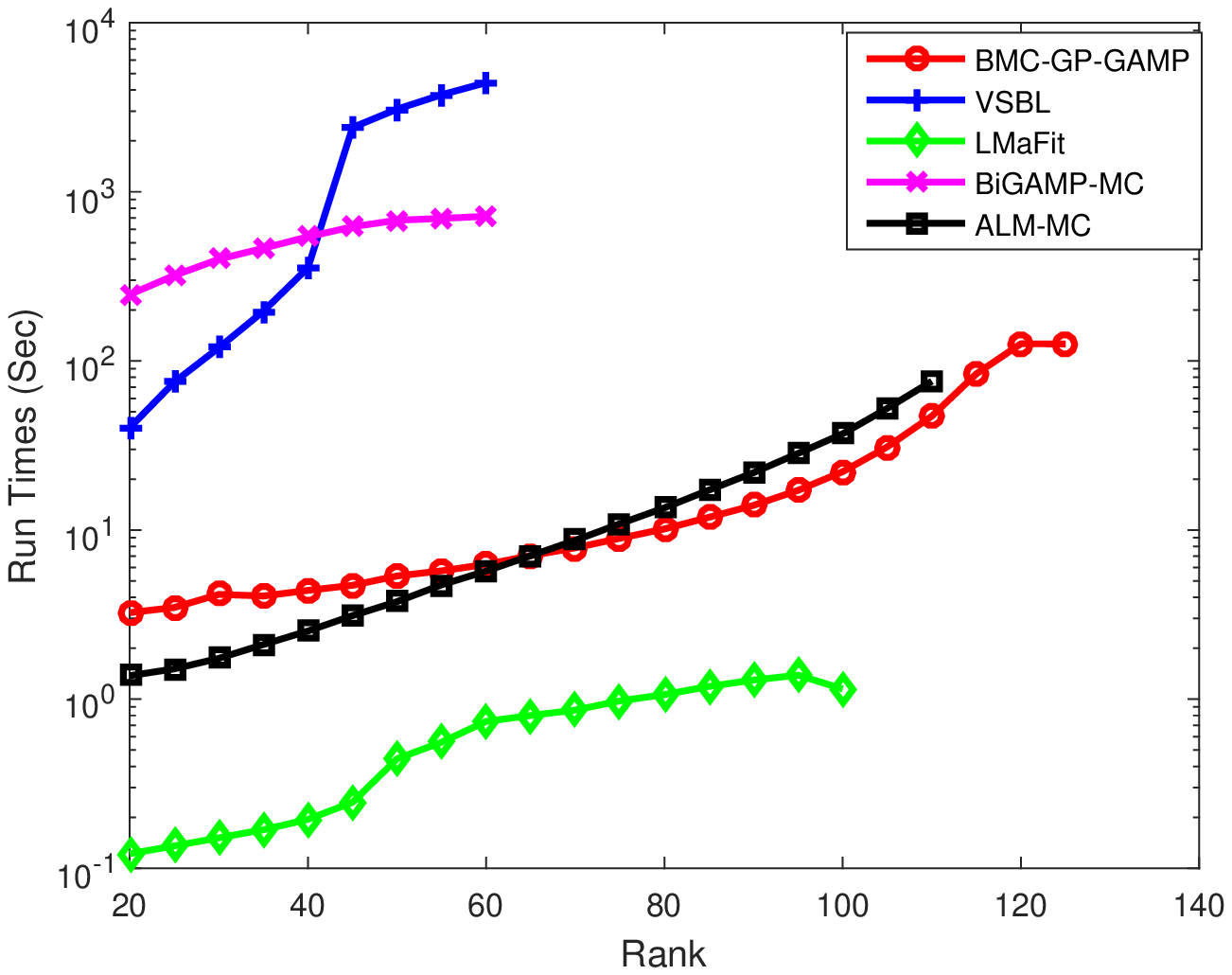}
    }
    \caption{Synthetic data: (a) Success rates vs.
    the rank of the matrix, $\rho=0.2$; (b) Run times vs. the rank of the matrix,
    $\rho=0.2$;
    (c) Success rates vs. the rank of the matrix $\rho=0.5$; (d) Run times vs. the rank of the matrix, $\rho=0.5$}
    \label{fig:succ-rate}
\end{figure*}

\subsection{Synthetic data}
We first examine the effectiveness of our proposed method on
synthetic data. We generate the test rank-$k$ matrix
$\boldsymbol{X}$ of size $500\times500$ by multiplying
$\boldsymbol{A}\in\mathbb{R}^{500\times k}$ by the transpose of
$\boldsymbol{B}\in\mathbb{R}^{500\times k}$, i.e.
$\boldsymbol{X}=\boldsymbol{A}\boldsymbol{B}^T$. All the entries
of $\boldsymbol{A}$ and $\boldsymbol{B}$ are sampled from a normal
distribution. We consider the scenarios where $20\%$ ($\rho=0.2$)
and $50\%$ ($\rho=0.5$) entries of $\boldsymbol{X}$ are observed.
Here $\rho$ denotes the sampling ratio. The success rates as well
as the run times of respective algorithms as a function of the
rank of $\boldsymbol{X}$, i.e. $k$, are plotted in Fig.
\ref{fig:succ-rate} Results are averaged over 25 independent
trials. A trial is considered to be successful if the relative
error is smaller than $10^{-2}$, i.e.
$||\boldsymbol{X}-\boldsymbol{\hat{X}}||_F/||\boldsymbol{X}||_F<10^{-2}$,
where $\boldsymbol{\hat{X}}$ denotes the estimated matrix. For our
proposed method, the matrix parameter $\boldsymbol{W}$ is set to
$10^{10}\boldsymbol{I}$. The pre-defined overestimated rank for
VSBL and LMaFit is set to be twice the true rank. For the case
$\rho=0.2$, VSBL and BiGAMP-MC present the same recovery
performance with their curves overlapping each other. From Fig.
\ref{fig:succ-rate}, we can see that
\begin{itemize}
    \item[1)] Our proposed method presents the best performance
    for both sampling ratio cases. Meanwhile, it has a moderate computational complexity.
    When the sampling ratio is
    set to 0.5, our proposed method has a run time similar to the
    ALM-MC method, while provides a clear performance improvement
    over the ALM-MC.
    \item[2)] The LMaFit method is the most computationally
    efficient. But its performance is not as good as our proposed
    method.
    \item[3)] The proposed method outperforms the other two Bayesian methods, namely, the
    VSBL and the BiGAMP-MC, by a big margin in terms of
    both recovery accuracy and computational complexity. Since the BiGAMP cannot automatically
    determine the matrix rank, it needs to try all possible values of the rank,
    which makes running the BiGAMP-MC time costly.
\end{itemize}




\begin{table}
    \centering
    \caption{Error Rate for Chr22 Dataset}
    \begin{tabular}{ccc}
        \hline
        \hline
        & $20\%$ & $50\%$\\
        \hline
        BMC-GP-GAMP & 0.0567 & $\textbf{0.0233}$ \\
        \hline
        VSBL & 0.0587 & 0.0249 \\
        LMaFit & 0.2525 & 0.2472 \\
        BiGAMP-MC & 0.0573 & 0.0282 \\
        ALM-MC & $\textbf{0.0550}$ & 0.0246 \\
        \hline
    \end{tabular}
    \label{table1}
\end{table}

\subsection{Gene data}
We carry out experiments on gene data for genotype estimation. The
dataset \cite{JiangMa2016} is a matrix of size $790\times112$
provided by Wellcome Trust Case Control Consortium (WTCCC) and
contains the genetic information from chromosome 22. The dataset,
which is referred to as ``Chr22'', has been shown in
\cite{JiangMa2016} to be approximately low-rank. We randomly
select $20\%$ or $50\%$ of the entries of the dataset as
observations, and recover the rest entries using low-rank matrix
completion methods. Again, for our proposed method, the matrix
parameter $\boldsymbol{W}$ is set to $10^{10}\boldsymbol{I}$. The
pre-defined ranks used for VSBL and LMaFit are both set to $100$.
Following \cite{JiangMa2016}, we use a metric termed as
``allelic-imputation error rate'' to evaluate the performance of
respective methods. The error rate is defined as
\begin{align}
\text{Error
Rate}=\frac{\text{nnz}(|\boldsymbol{X}-\text{round}(\hat{\boldsymbol{X}})|)}{T}
\end{align}
where $\boldsymbol{X}$ and $\hat{\boldsymbol{X}}$ denotes the true
and the estimated matrices, respectively, the operation
$\text{round}(\boldsymbol{X})$ returns a matrix with each entry of
$\boldsymbol{X}$ rounded to its nearest integer,
$\text{nnz}(\boldsymbol{X})$ counts the number of non-zero entries
of $\boldsymbol{X}$, and $T$ denotes the number of unobserved
entries. We report the average error rates of respective
algorithms in Table \ref{table1}. From Table \ref{table1}, we see
that all methods, except the LMaFit method, present similar
results and the proposed method slightly outperforms other methods
when $50\%$ entries are observed. Despite the superior performance
on synthetic data, the LMaFit method incurs large estimation
errors for this dataset.

\begin{table}
    \centering
    \caption{NMAE for 100k Movielens Dataset}
    \begin{tabular}{ccc}
        \hline
        \hline
        & $20\%$ & $50\%$\\
        \hline
        BMC-GP-GAMP & $\textbf{0.1931}$ & 0.1851 \\
        \hline
        VSBL & 0.2004 & $\textbf{0.1847}$ \\
        LMaFit & 0.2677 & 0.2354 \\
        BiGAMP-MC & 0.2009 & 0.1856 \\
        ALM-MC & 0.2002 & 0.1893 \\
        \hline
    \end{tabular}
    \label{table2}
\end{table}

\subsection{Collaborative Filtering}
In this experiment, we study the performance of respective methods
on the task of collaborative filtering. We use the MovieLens 100k
dataset\footnote{Available at \url{http://www.grouplens.org/node/73/}},
which consists of $10^5$ ratings ranging from 1 to 5 on
1682 movies from 943 users. The ratings can form a matrix of size
$943\times 1682$. We randomly choose $20\%$ or $50\%$ of available
ratings as training data, and predict the rest ratings using
respective matrix completion methods. The matrix parameter
$\boldsymbol{W}$ used in the proposed method is set to
$10^{10}\boldsymbol{I}$. The pre-defined ranks used for VSBL and
LMaFit are both set to $100$. The performance is evaluated by the
normalized mean absolute error (NMAE), which is calculated as
\begin{align}
\text{NMAE}=\frac{\sum_{(i,j)\in
S}|x_{ij}-\hat{x}_{ij}|}{(r_{\text{max}}-r_{\text{min}})|S|}
\end{align}
where $S$ is a set containing the indexes of those unobserved
available ratings, $r_{\text{max}}$ and $r_{\text{min}}$ denote
the maximum and minimum ratings, respectively. The results of NMAE
are shown in Table \ref{table2}, from which we see that the
proposed method achieves the most accurate rating prediction when
the number of observed ratings is small.


\begin{table*}
    \centering
    \caption{Image Inpainting (PSNR/SSIM)}
    \begin{tabular}{ccccc}
        \hline
        \hline
        &\multicolumn{2}{c}{Monarch}&\multicolumn{2}{c}{Lena}\\
        \cmidrule(lr){2-3}
        \cmidrule(lr){4-5}
        & $30\%$ & $50\%$ &  $20\%$ & $30\%$\\
        \hline
        BMC-GP-GAMP-I & 19.3328/0.4942 & 23.8965/0.7066 & 23.4235/0.4946 & 25.6866/0.5991 \\
        BMC-GP-GAMP-II & 20.8628/$\textbf{0.6960}$ & $\textbf{25.6955}$/$\textbf{0.8705}$ &
        $\textbf{25.3337}$/$\textbf{0.7479}$ & $\textbf{28.0434}$/$\textbf{0.8328}$ \\
        BMC-GP-GAMP-III & $\textbf{21.1797}$/0.6710 & 25.6665/0.8422 & 25.2876/0.7388 & 27.9684/0.8213 \\
        \hline
        VSBL & 16.5478/0.3625 & 19.9965/0.5468 & 21.5168/0.4999 & 23.8180/0.5964 \\
        LMaFit & 17.6527/0.4018 & 19.2834/0.5117 & 21.5525/0.4494 & 22.3862/0.5504 \\
        BiGAMP-MC & 18.9154/0.4618 & 22.3883/0.6338 & 22.6557/0.4869 & 24.8173/0.5997 \\
        ALM-MC & 19.6250/0.5149 & 23.6854/0.7253 & 23.1028/0.5164 & 25.5310/0.6433 \\
        \hline
    \end{tabular}
    \label{table3}
\end{table*}



\subsection{Image Inpainting}
Lastly, we evaluate the performance of different methods on image
inpainting. The objective of image inpainting is to complete an
image with missing pixels. We conduct experiments on the benchmark
images Butterfly and Lena, which are of size $256\times 256$ and
$512\times 512$, respectively. In our experiments, we examine the
performance of our proposed method under different choices of
$\boldsymbol{W}$. As usual, we can set
$\boldsymbol{W}=10^{10}\boldsymbol{I}$. Such a choice of
$\boldsymbol{W}$ is referred to as BMC-GP-GAMP-I. We can also set
$\boldsymbol{W}$ according to (\ref{W-1}) and (\ref{W-2}), which
are respectively referred to as BMC-GP-GAMP-II and
BMC-GP-GAMP-III. The parameters $\hat{\epsilon}$ and $\theta$ in
(\ref{W-2}) are set to $10^{-6}$ and $\sqrt{3}$, respectively. As
discussed earlier in our paper, the latter two choices exploit
both the low-rankness and the smoothness of the signal. For the
Butterfly image, we consider cases where $30\%$ and $50\%$ of
pixels in the image are observed. For the Lena image, we consider
cases where $20\%$ and $40\%$ of pixels are observed. We report
the peak signal to noise ratio (PSNR) as well as the structural
similarity (SSIM) index of each algorithm in Table \ref{table3}.
The original image with missing pixels and these images
reconstructed by respective algorithms are shown in Fig.
\ref{fig:monarch-03}, \ref{fig:monarch-05}, \ref{fig:lena-02}, and
\ref{fig:lena-03}. From Table \ref{table3}, we see that with a
common choice of $\boldsymbol{W}=10^{10}\boldsymbol{I}$, our
proposed method, BMC-GP-GAMP-I, outperforms other methods in most
cases. When $\boldsymbol{W}$ is more carefully devised, our
proposed method, i.e. BMC-GP-GAMP-II and BMC-GP-GAMP-III,
surpasses other methods by a substantial margin in terms of both
PSNR and SSIM metrics. This result indicates that a careful choice
of $\boldsymbol{W}$ that captures both the low-rank structure as
well as the smoothness of the latent matrix can help substantially
improve the recovery performance. From the reconstructed images,
we also see that our proposed method, especially BMC-GP-GAMP-II
and BMC-GP-GAMP-III, provides the best visual quality among all
these methods.


\begin{figure*}[t]
    \centering
    \includegraphics [width=115pt]{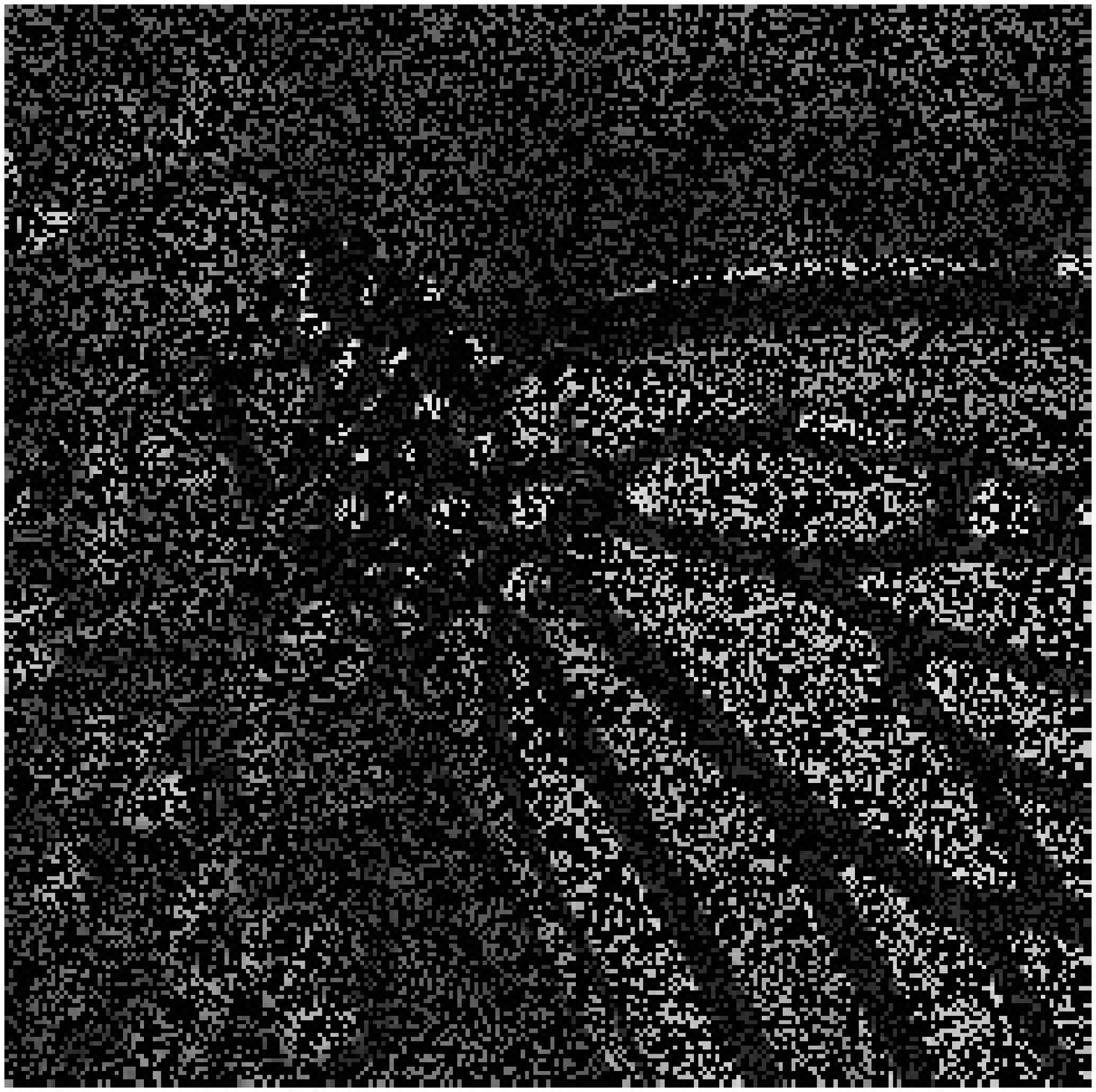}
    \includegraphics [width=115pt]{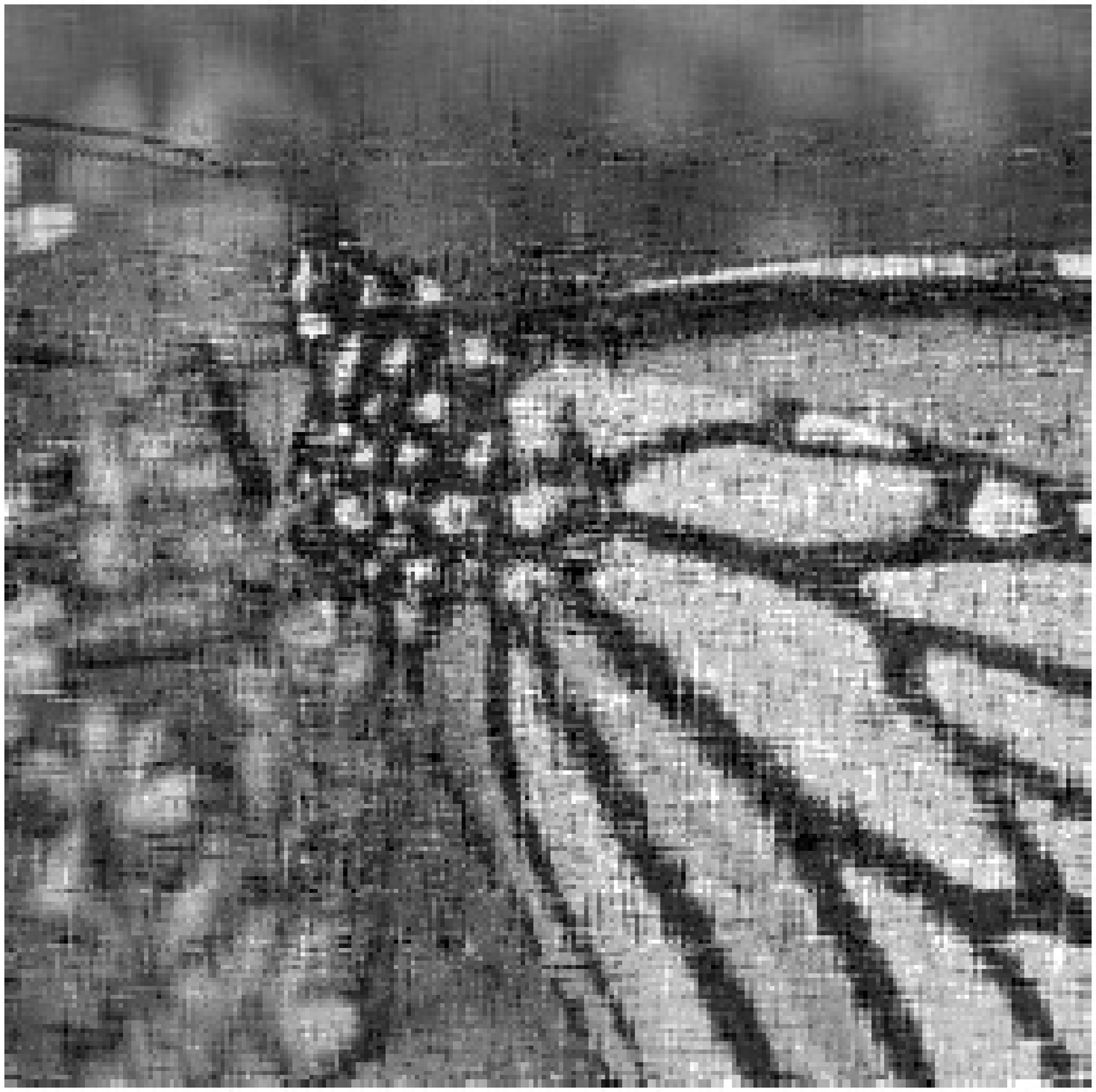}
    \includegraphics [width=115pt]{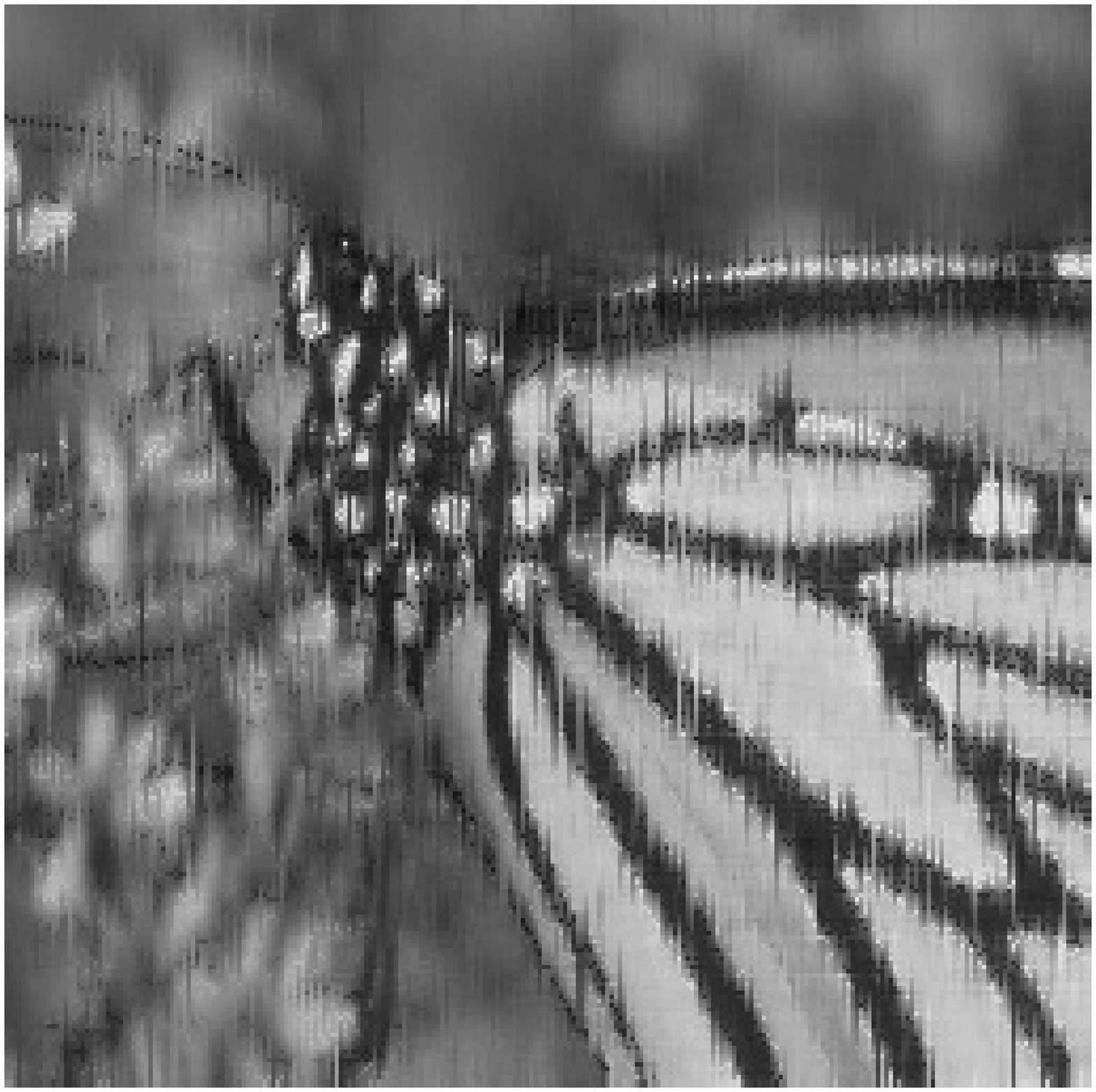}
    \includegraphics [width=115pt]{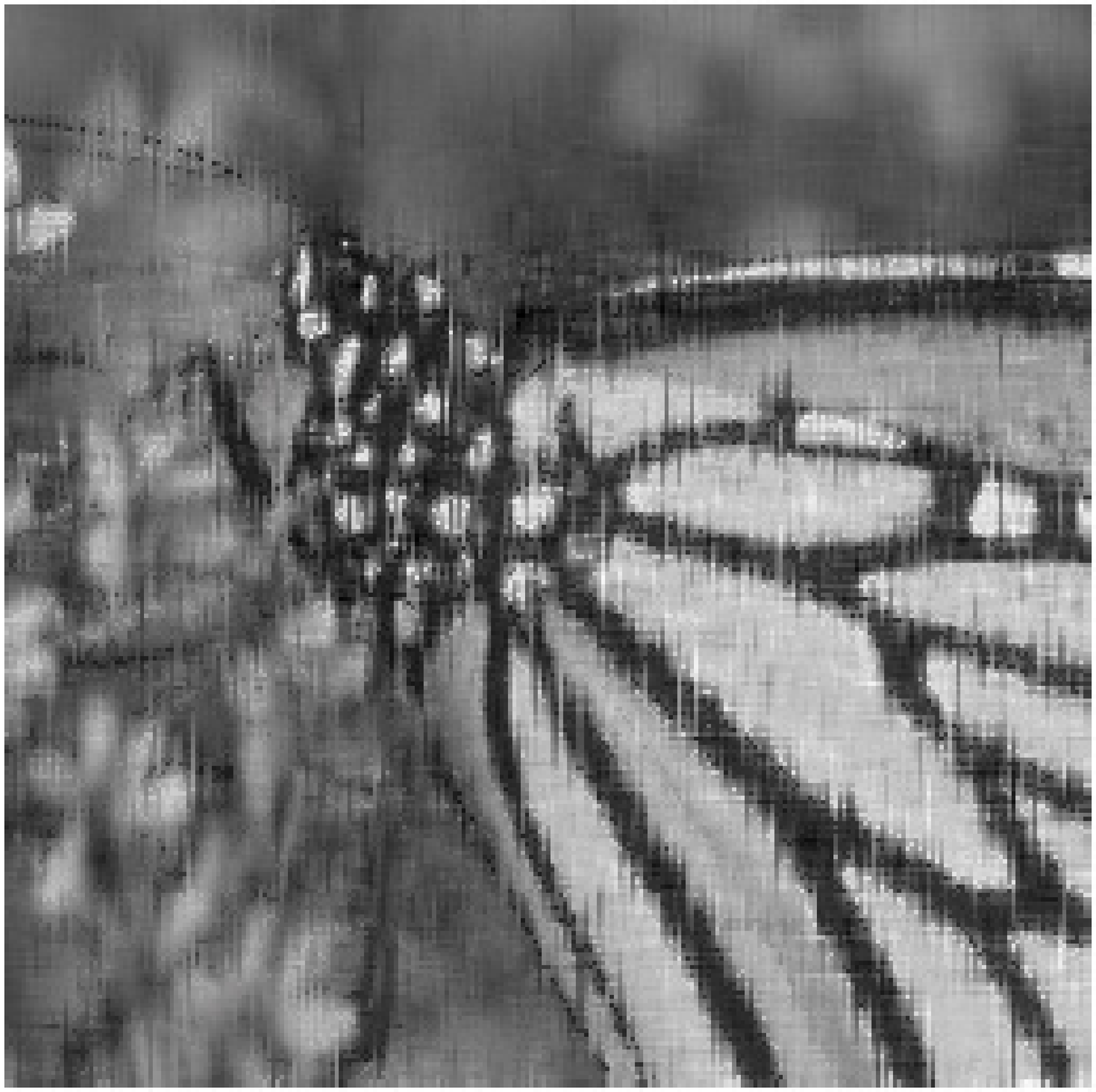}\\
    \includegraphics [width=115pt]{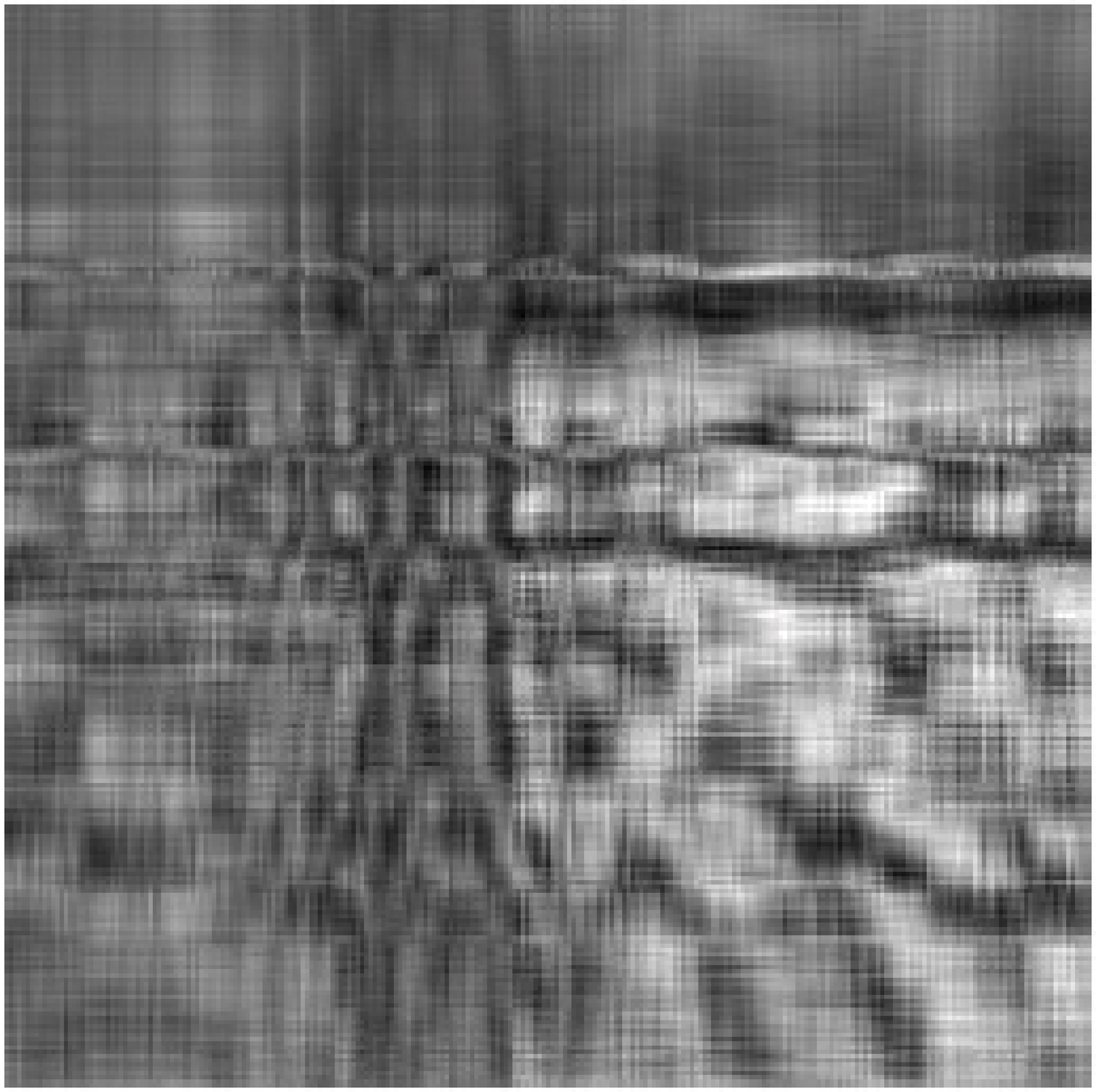}
    \includegraphics [width=115pt]{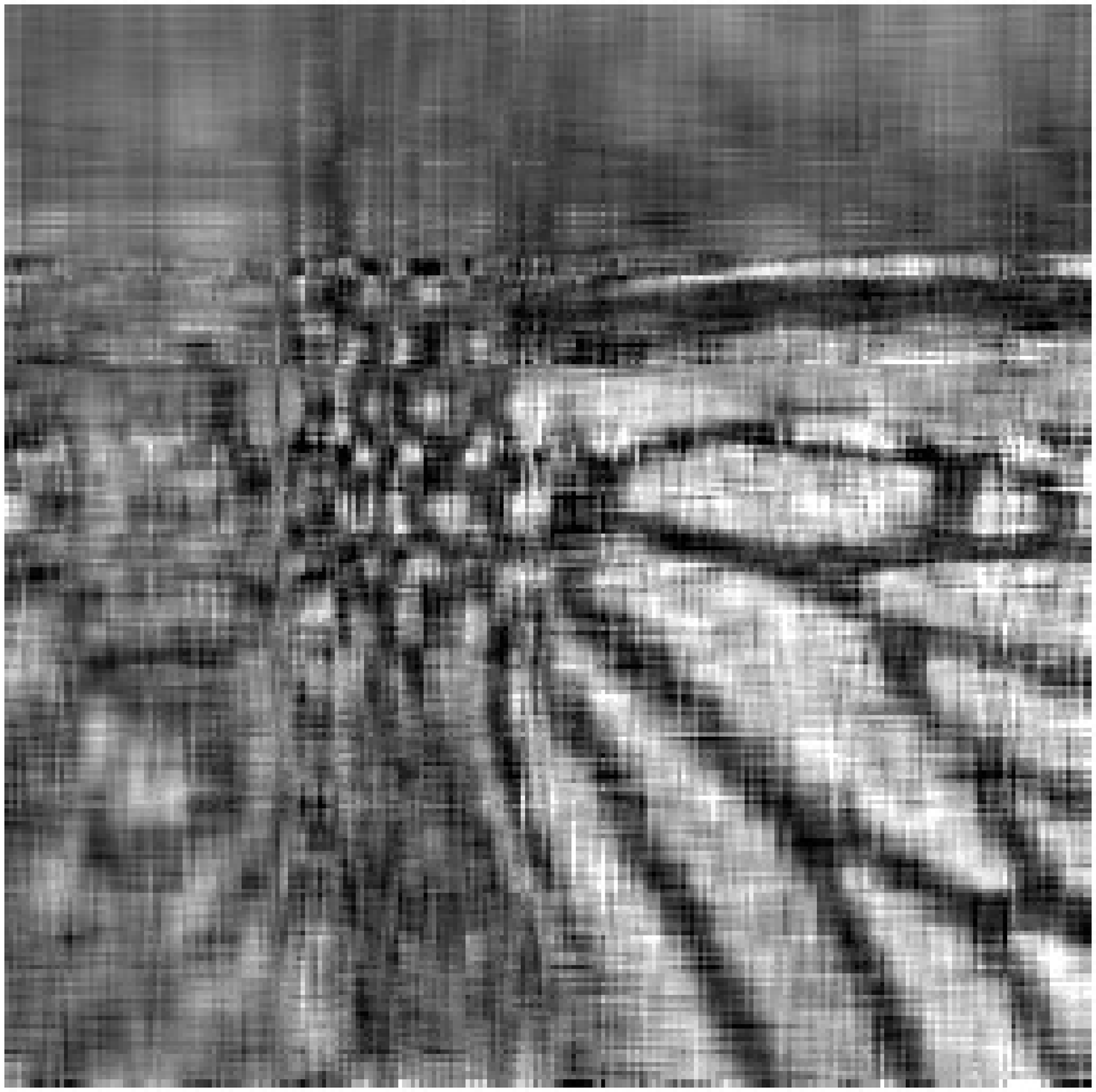}
    \includegraphics [width=115pt]{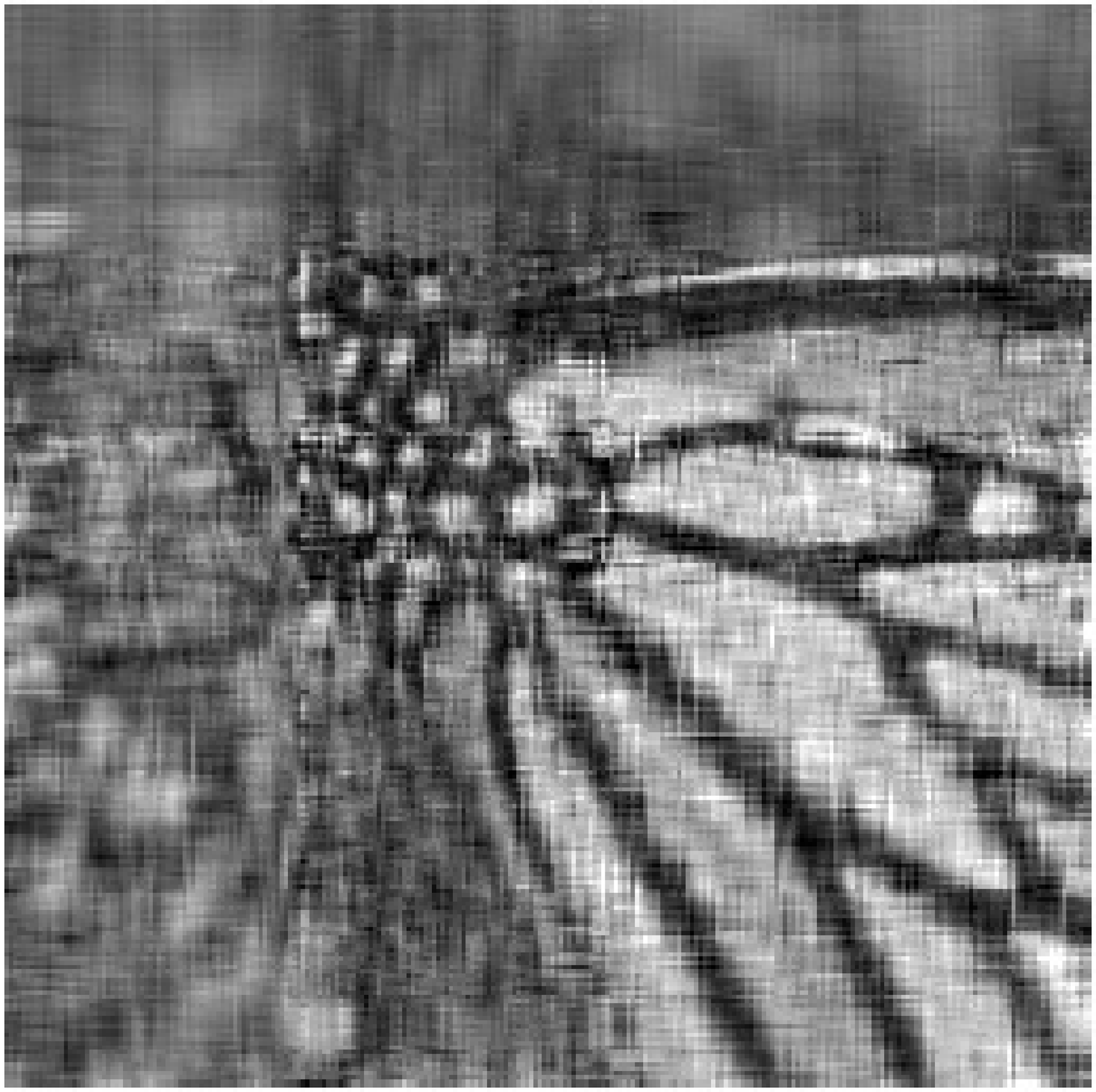}
    \includegraphics [width=115pt]{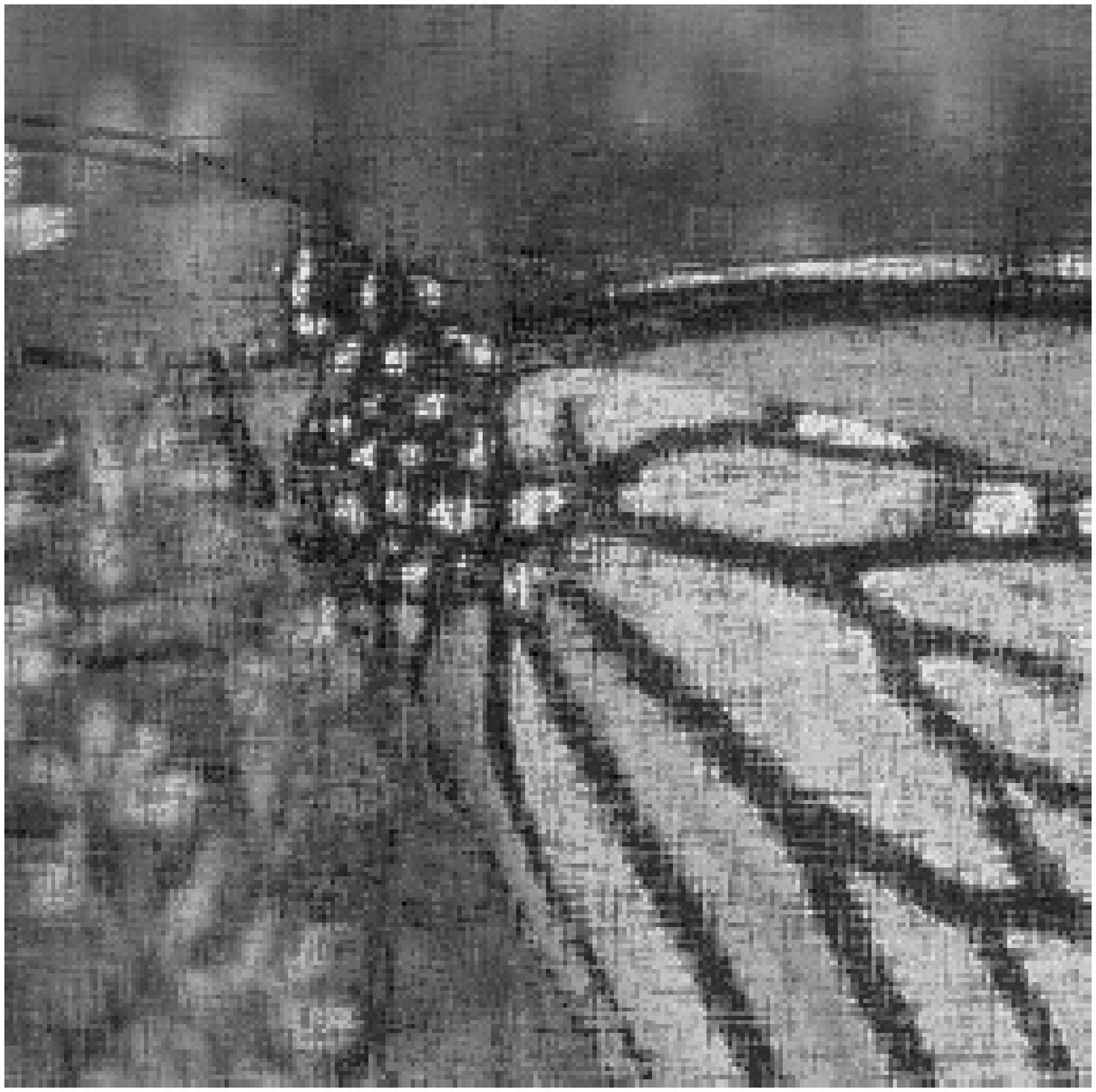}\\
    \caption{
    Top row (from left to right): observed Butterfly image with missing pixels ($\rho=0.3$),
    images recovered by BMC-GP-GAMP-I,
    BMC-GP-GAMP-II, and BMC-GP-GAMP-III, respectively. Bottom row (from left to right):
    images recovered by VSBL, LMaFit, BiGAMP-MC, and ALM-MC, respectively.}
    \label{fig:monarch-03}
\end{figure*}

\begin{figure*}[t]
    \centering
    \includegraphics [width=115pt]{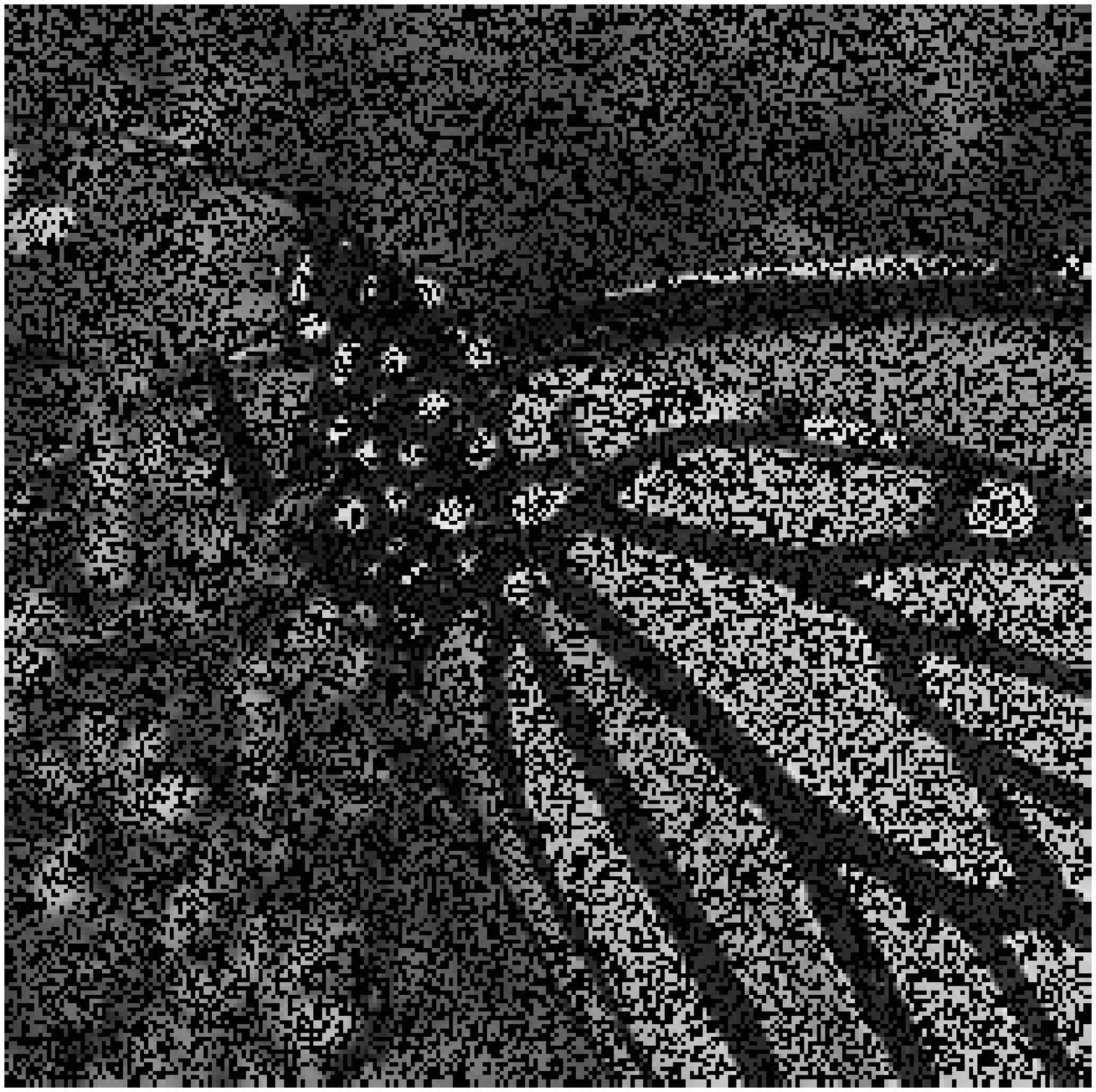}
    \includegraphics [width=115pt]{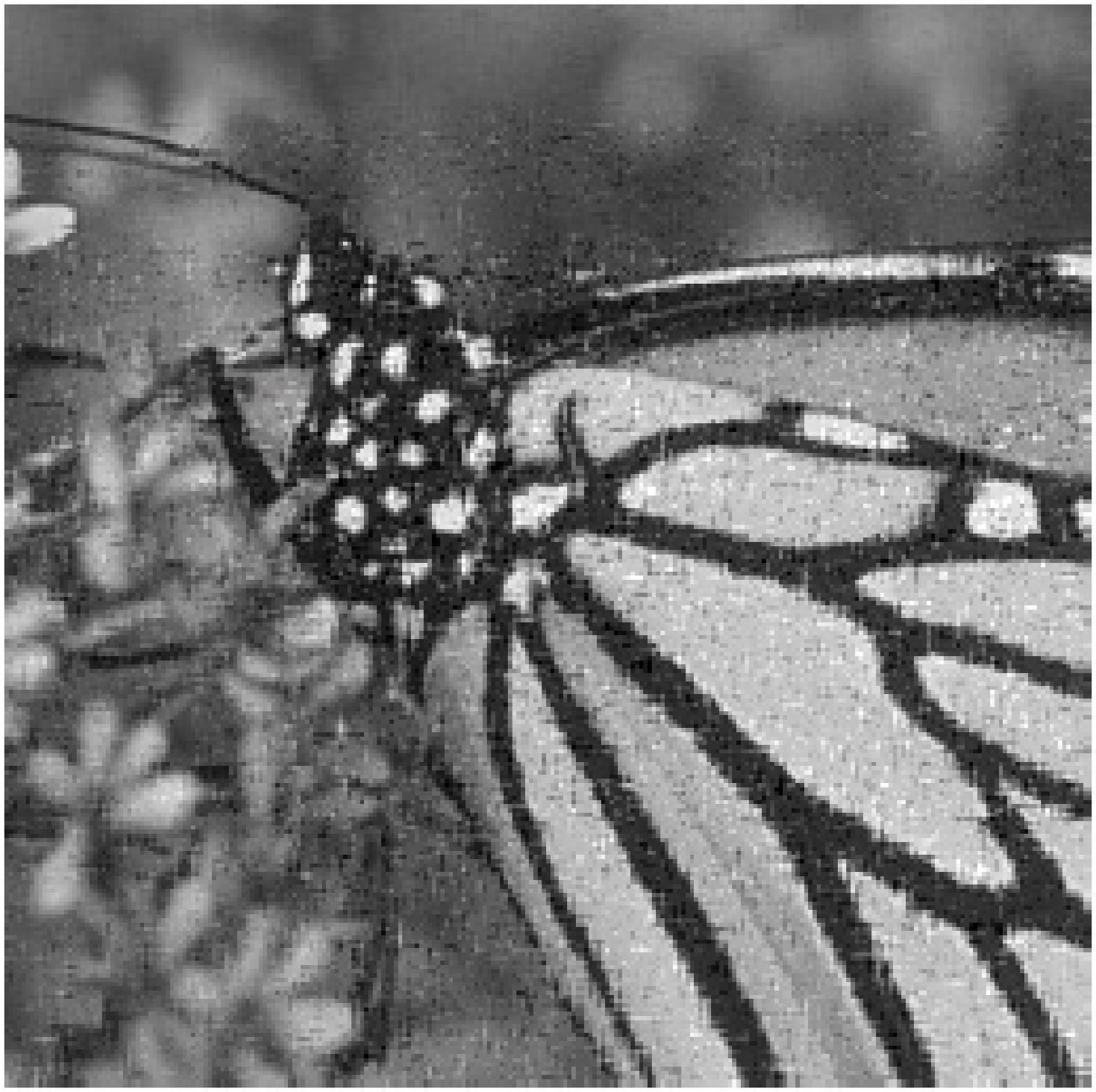}
    \includegraphics [width=115pt]{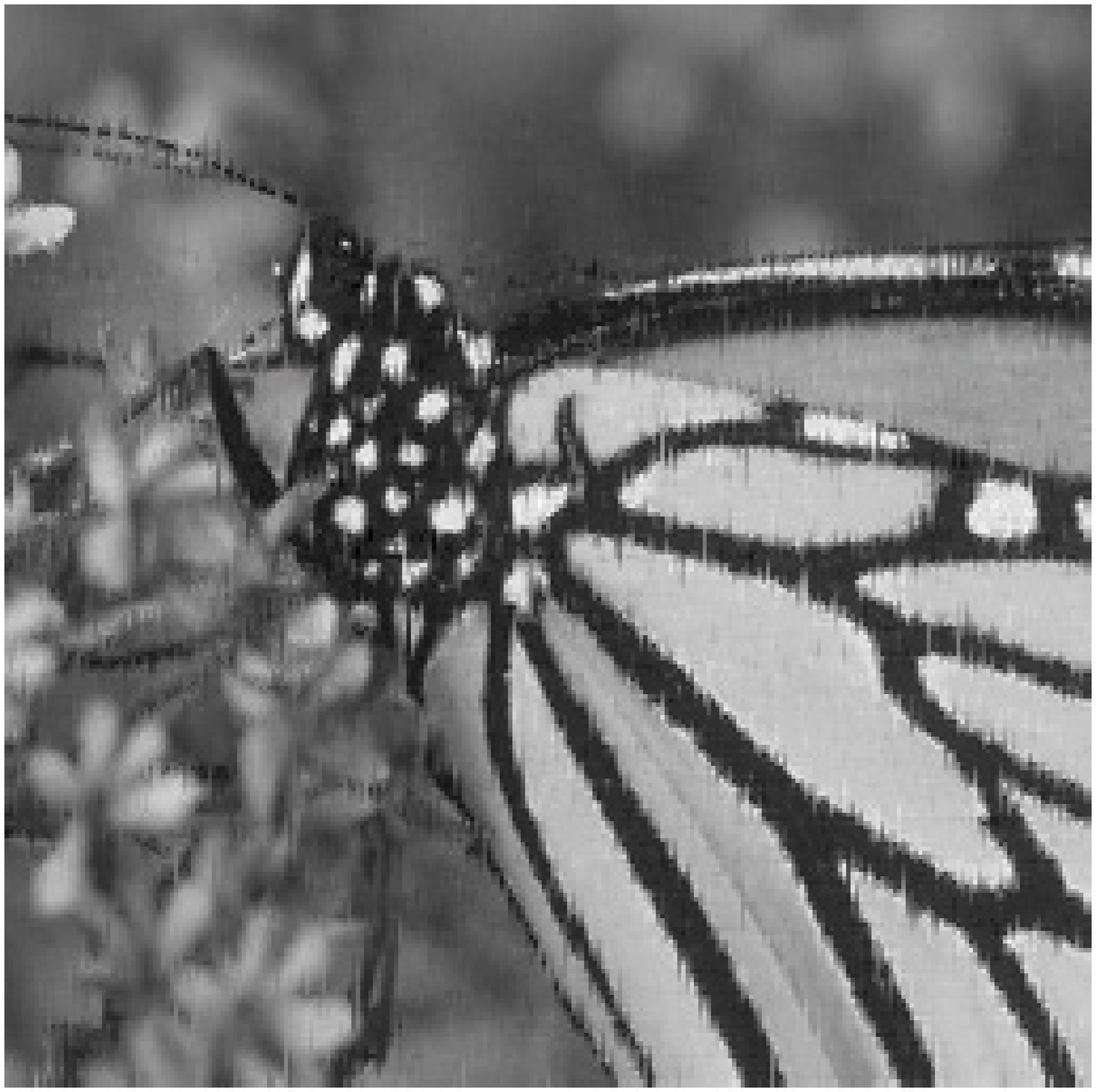}
    \includegraphics [width=115pt]{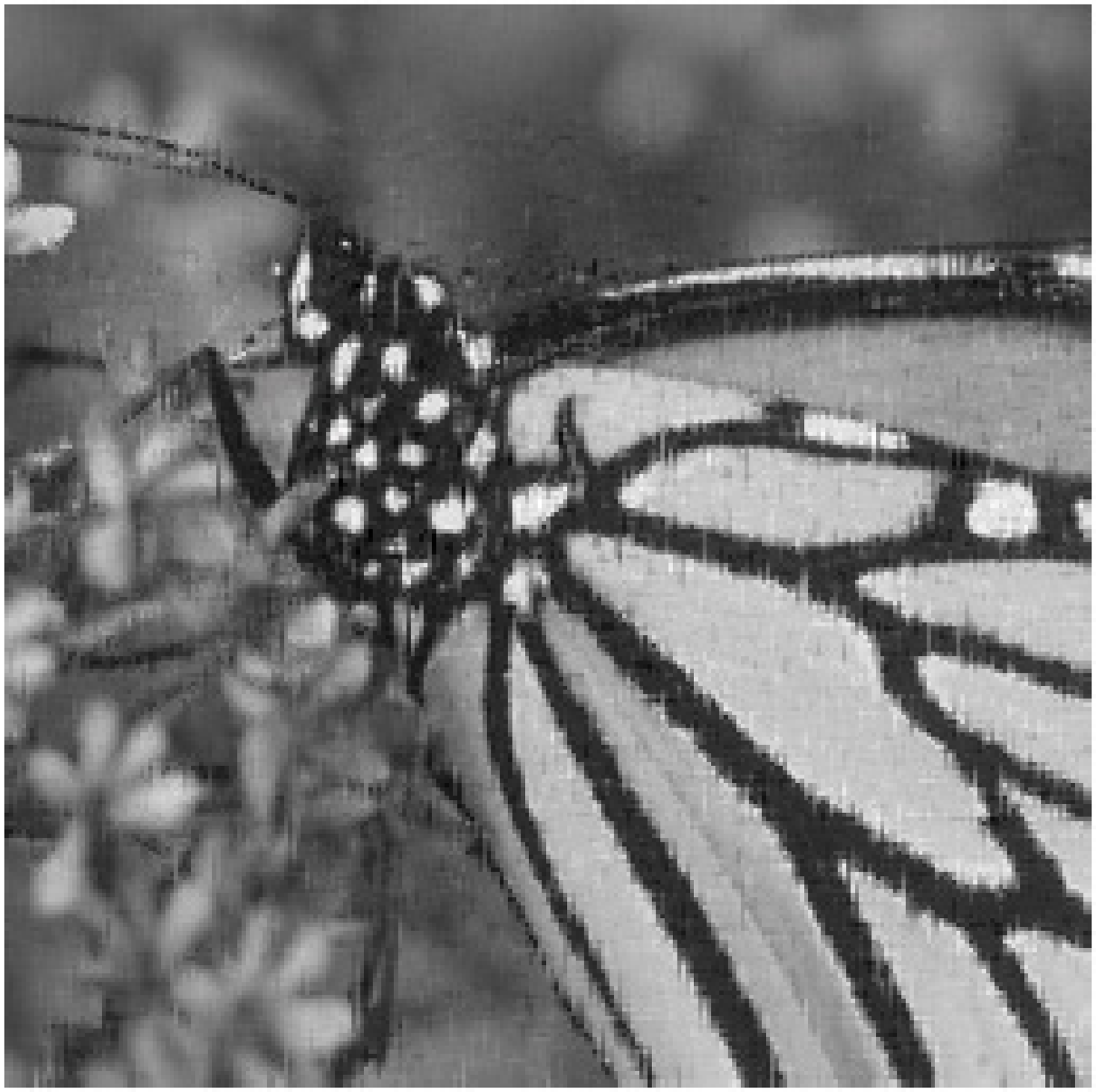}\\
    \includegraphics [width=115pt]{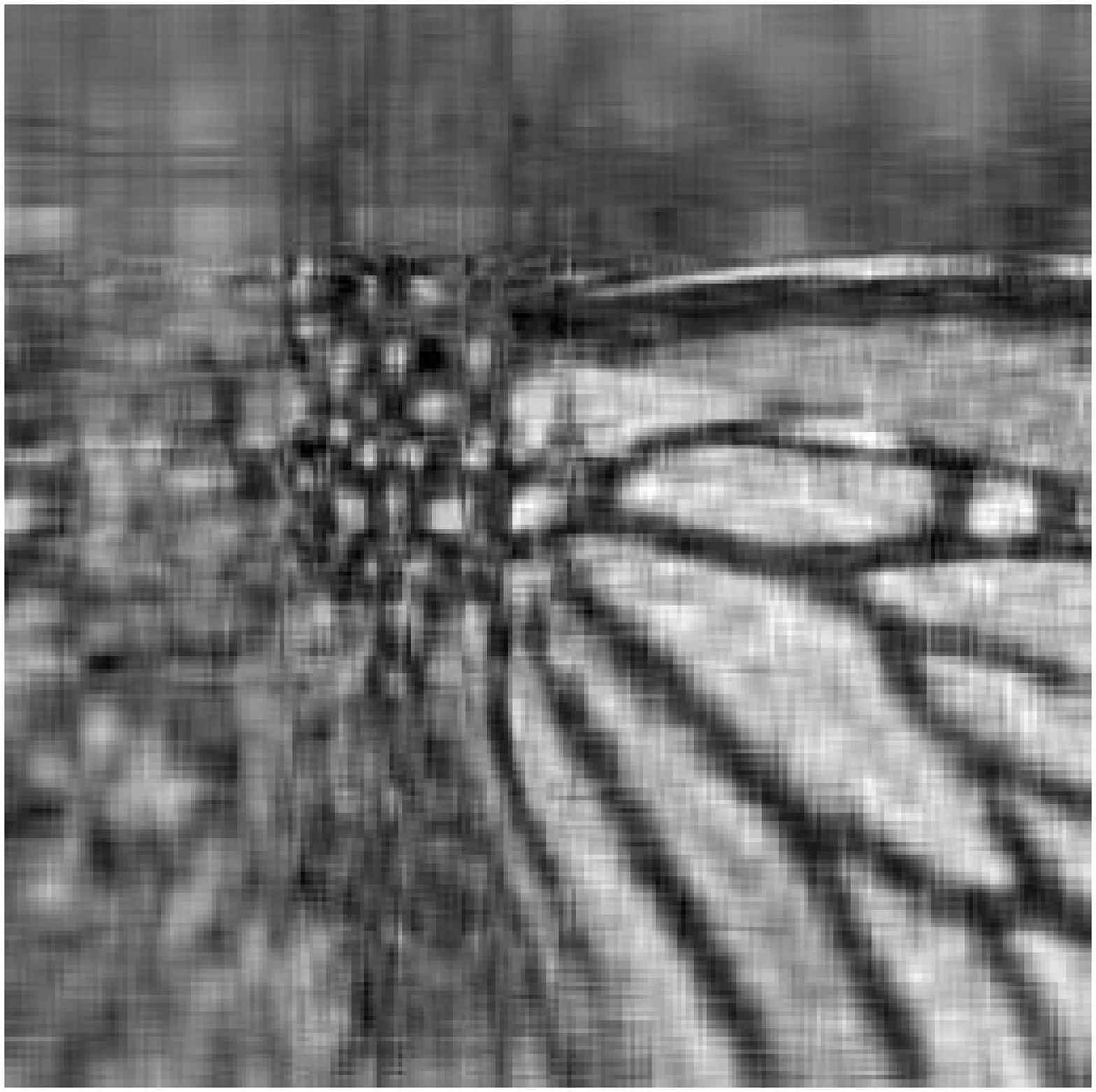}
    \includegraphics [width=115pt]{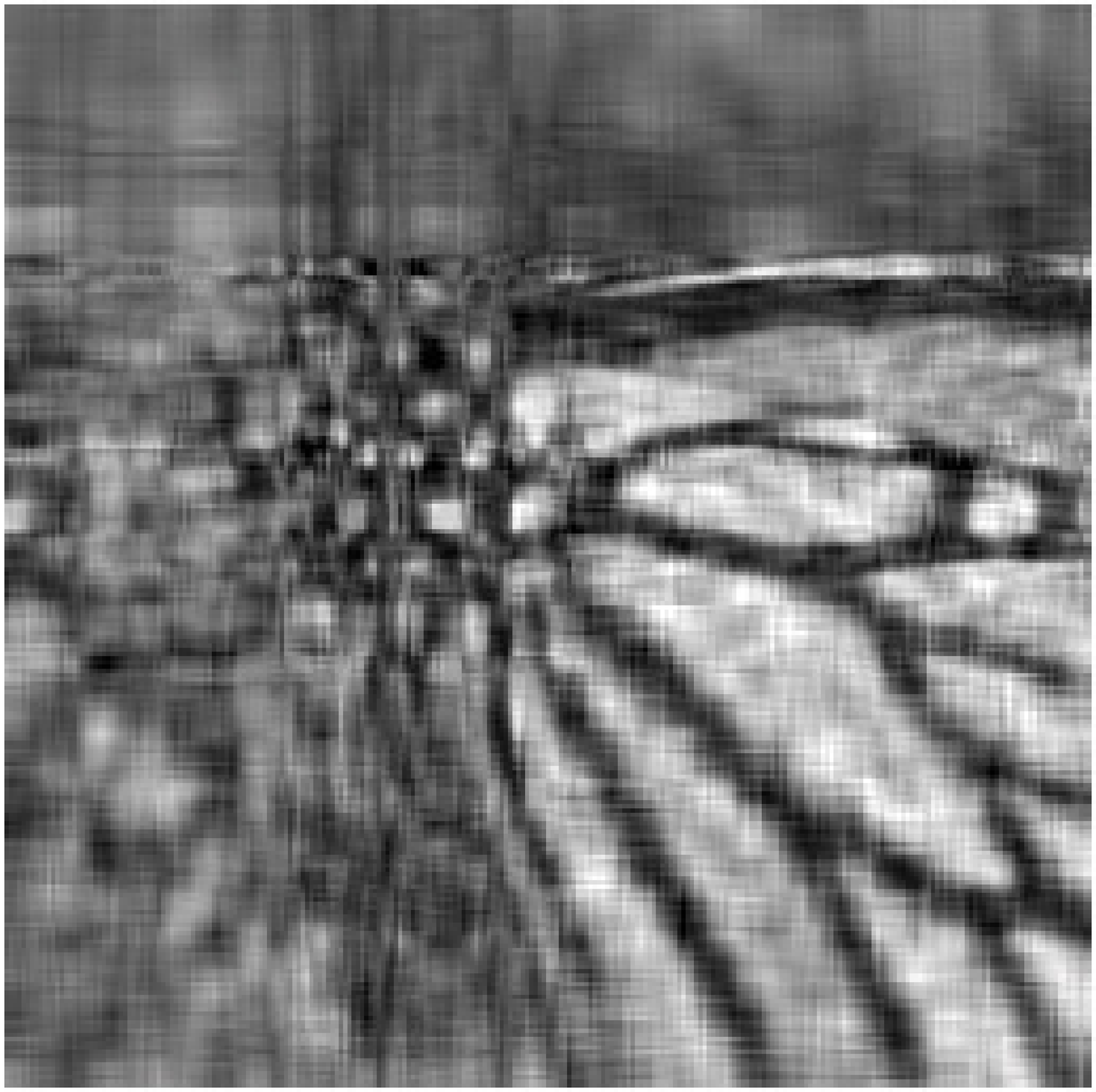}
    \includegraphics [width=115pt]{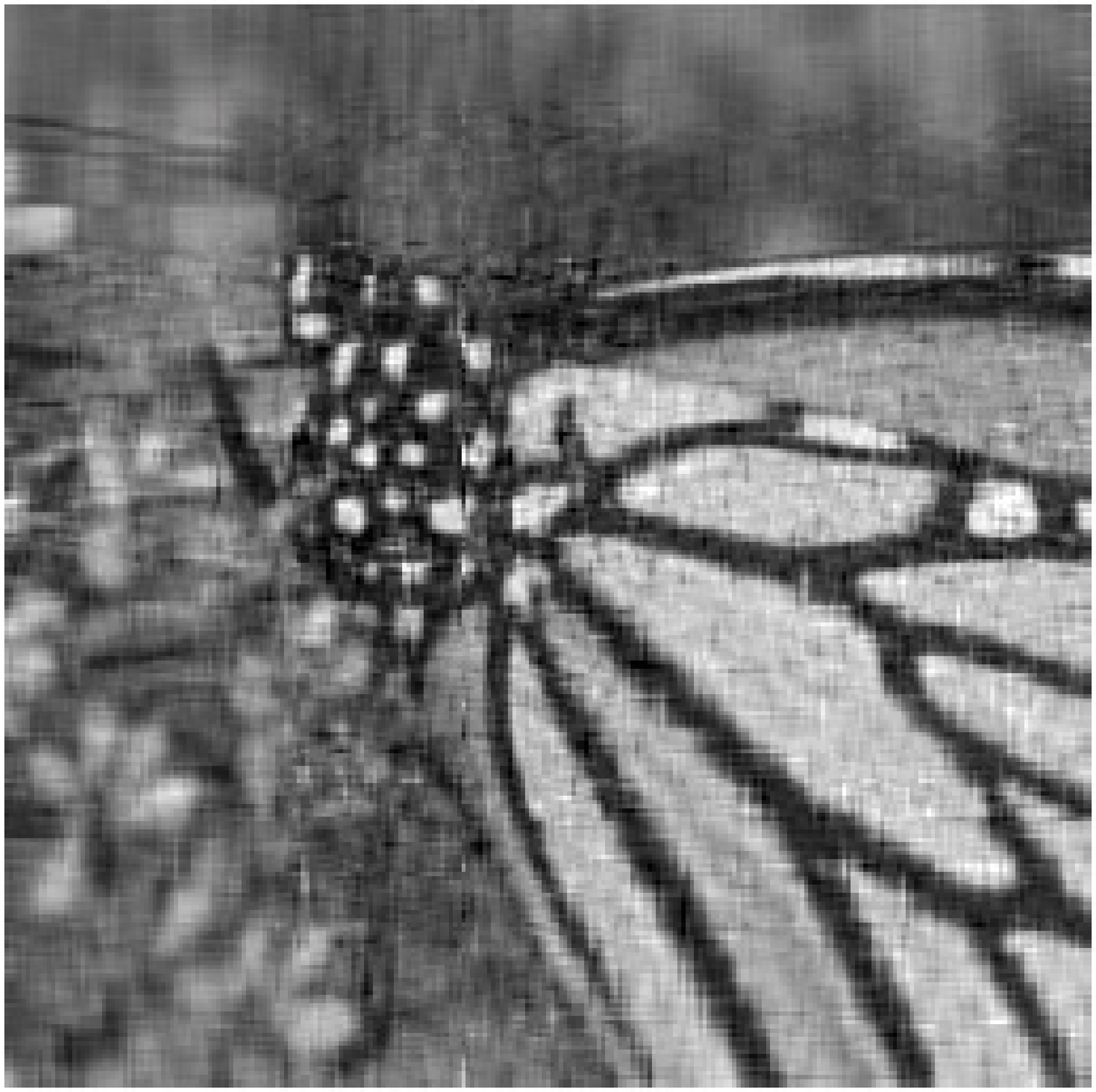}
    \includegraphics [width=115pt]{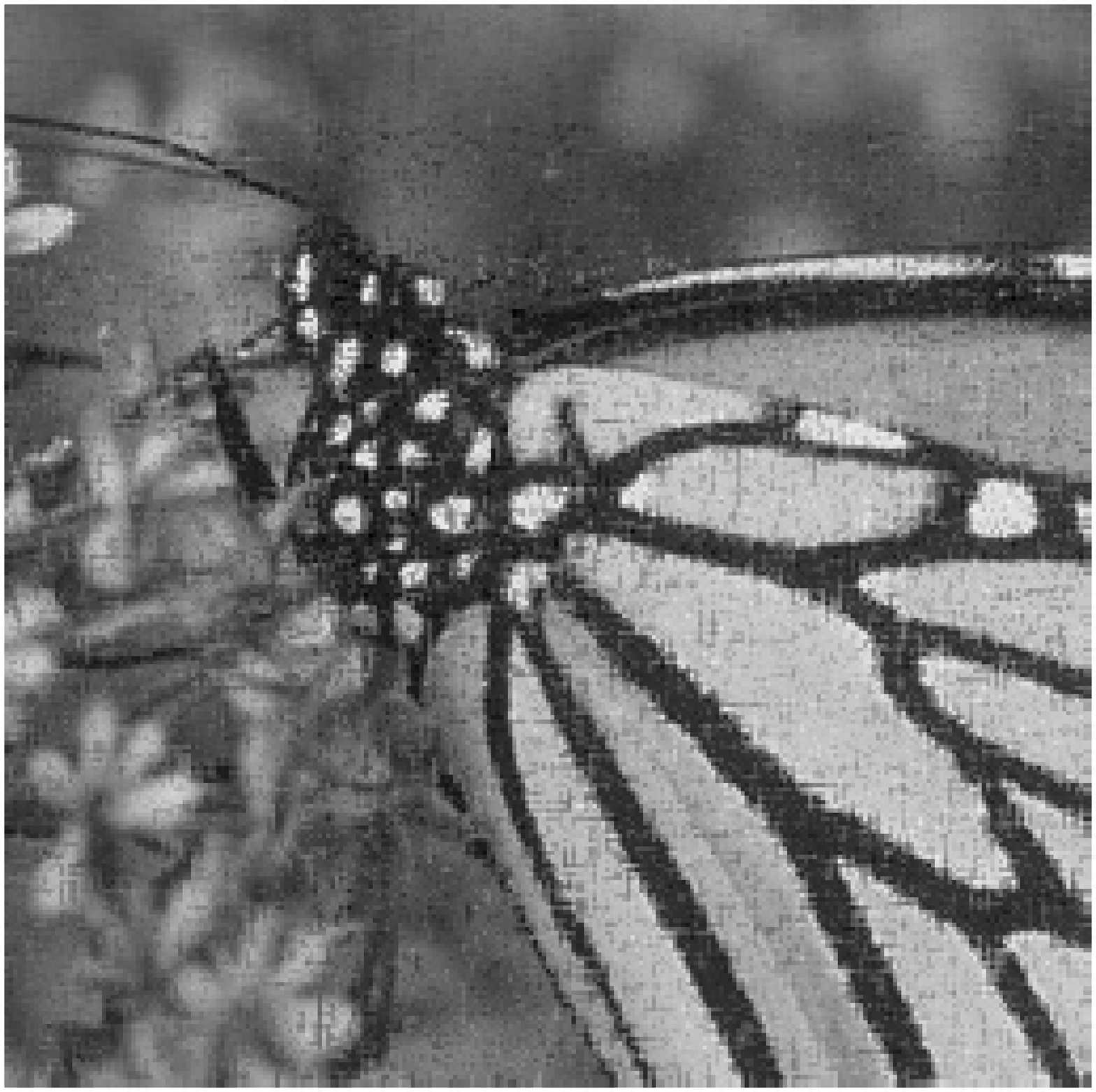}\\
    \caption{
    Top row (from left to right): observed Butterfly image
    with missing pixels ($\rho=0.5$), images recovered by BMC-GP-GAMP-I,
    BMC-GP-GAMP-II, and BMC-GP-GAMP-III, respectively. Bottom row (from left to right):
    images recovered by VSBL, LMaFit, BiGAMP-MC, and ALM-MC, respectively.}
    \label{fig:monarch-05}
\end{figure*}

\begin{figure*}[t]
    \centering
    \includegraphics [width=120pt]{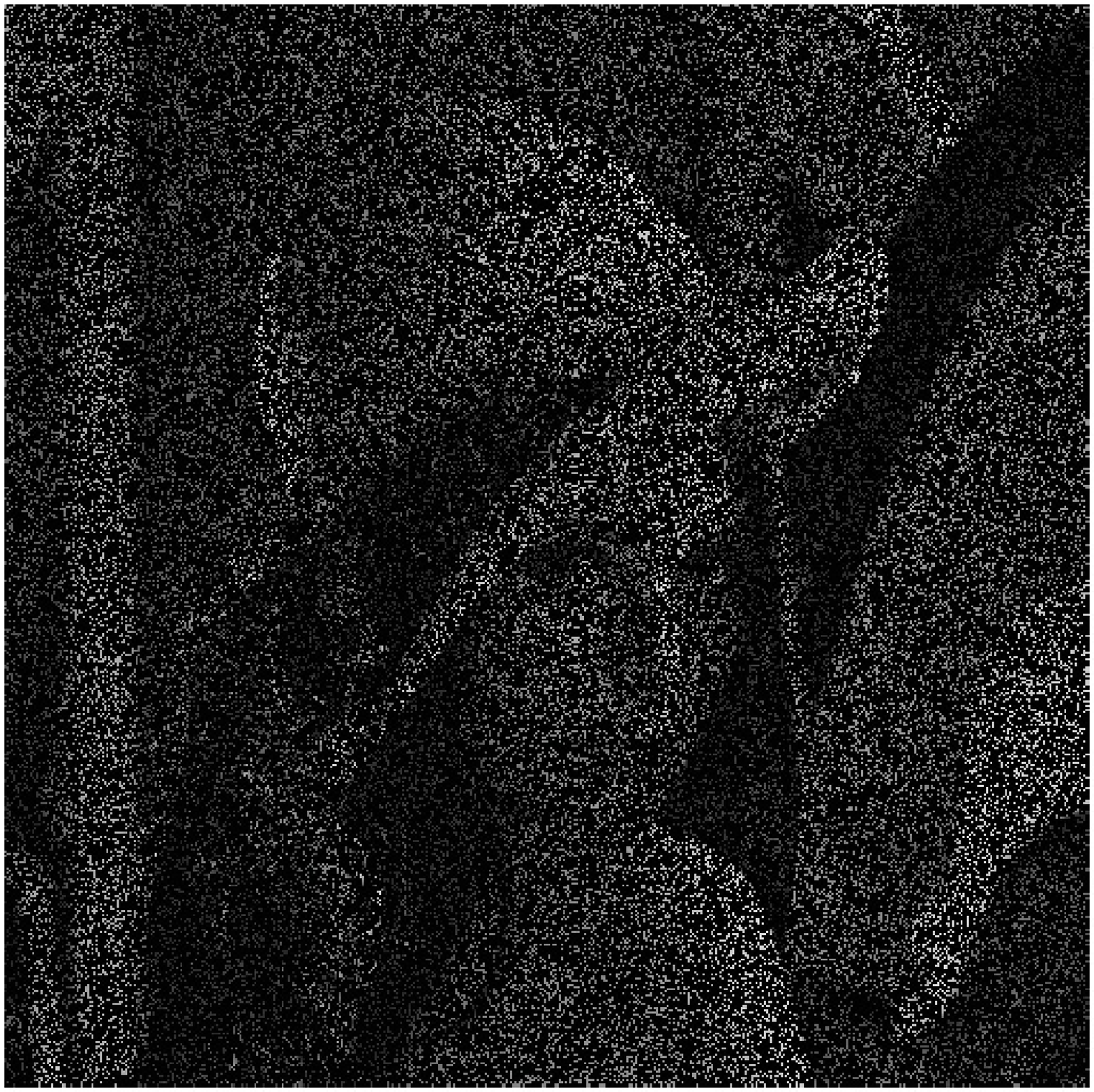}
    \includegraphics [width=120pt]{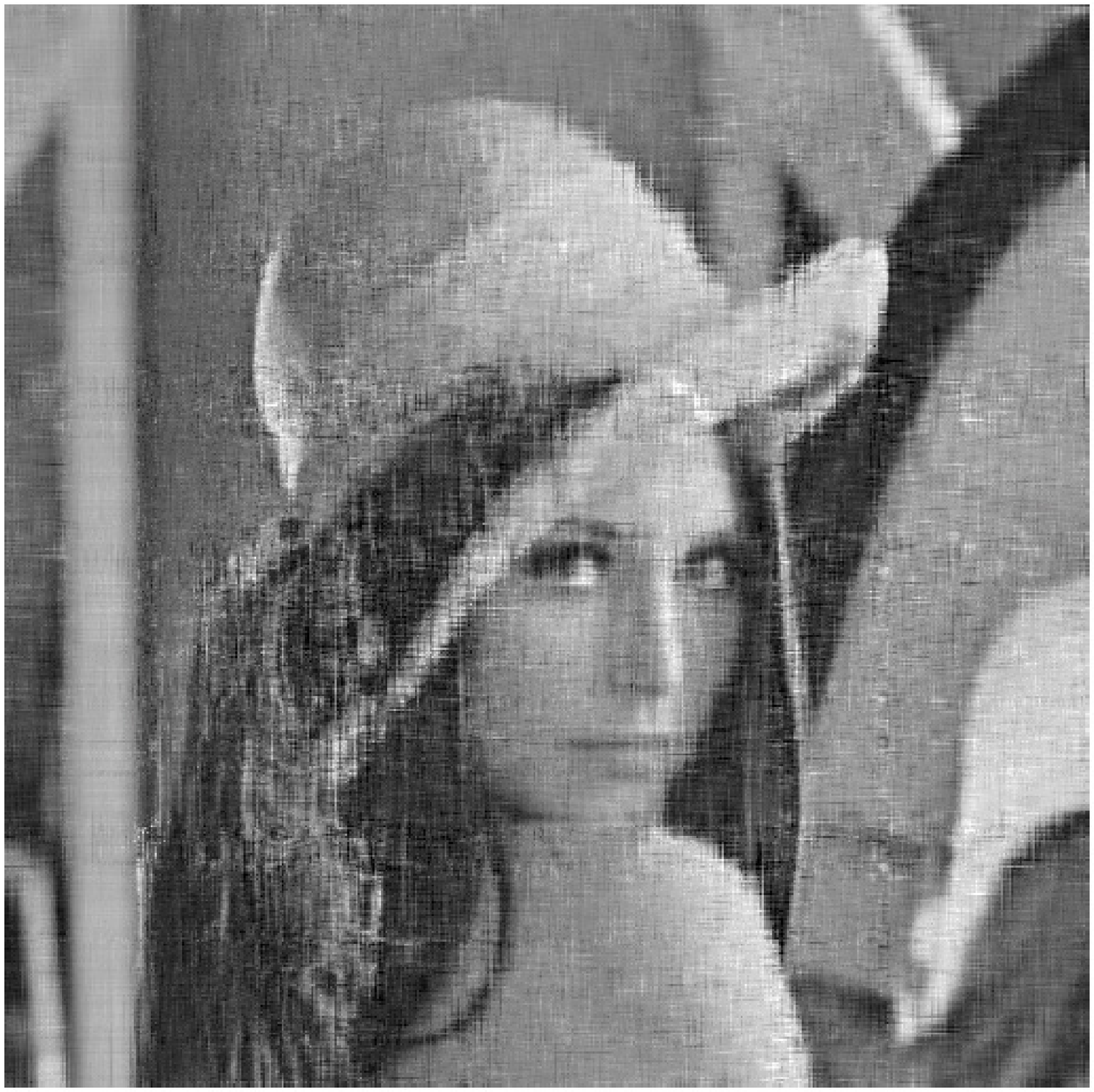}
    \includegraphics [width=120pt]{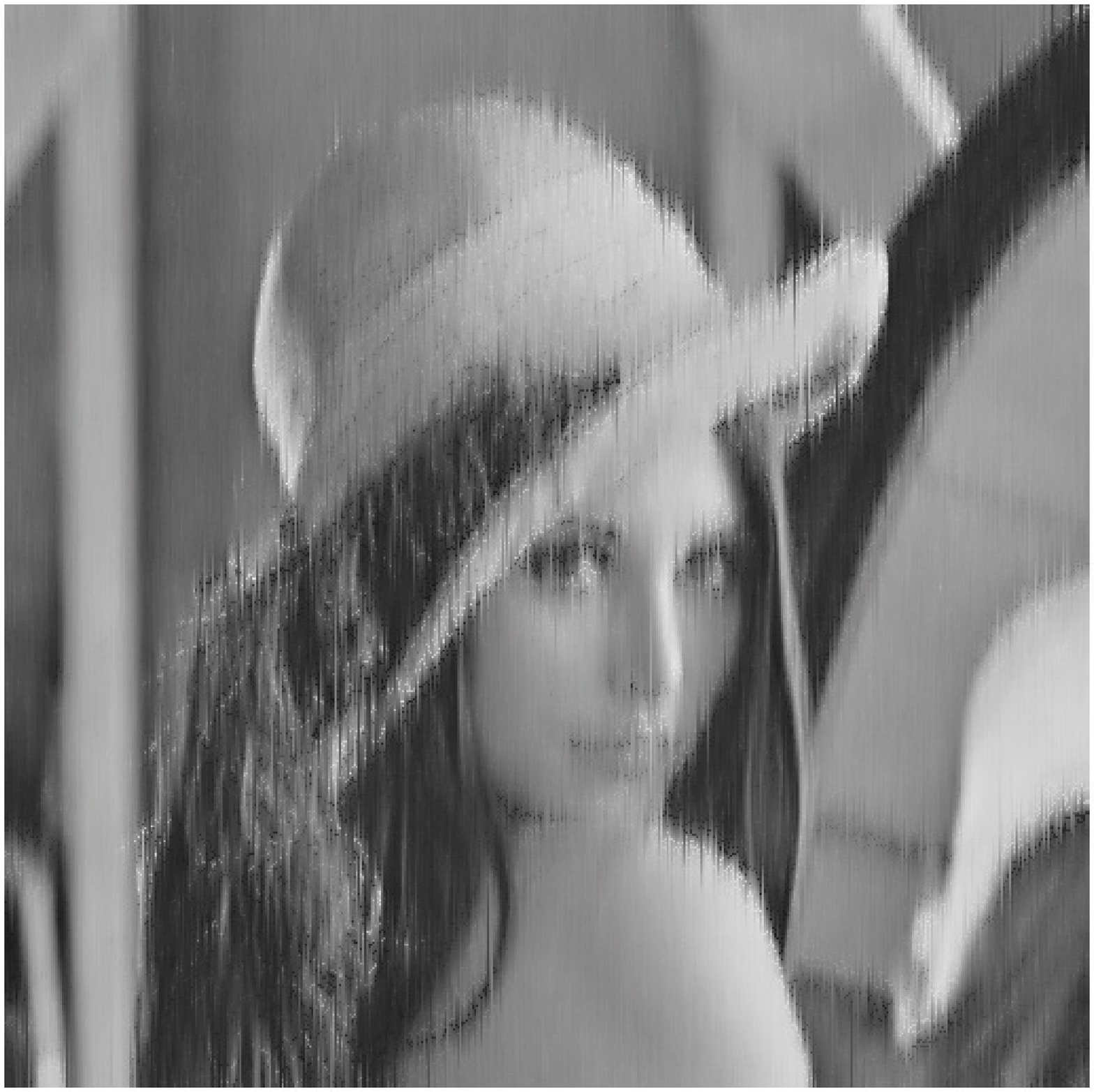}
    \includegraphics [width=120pt]{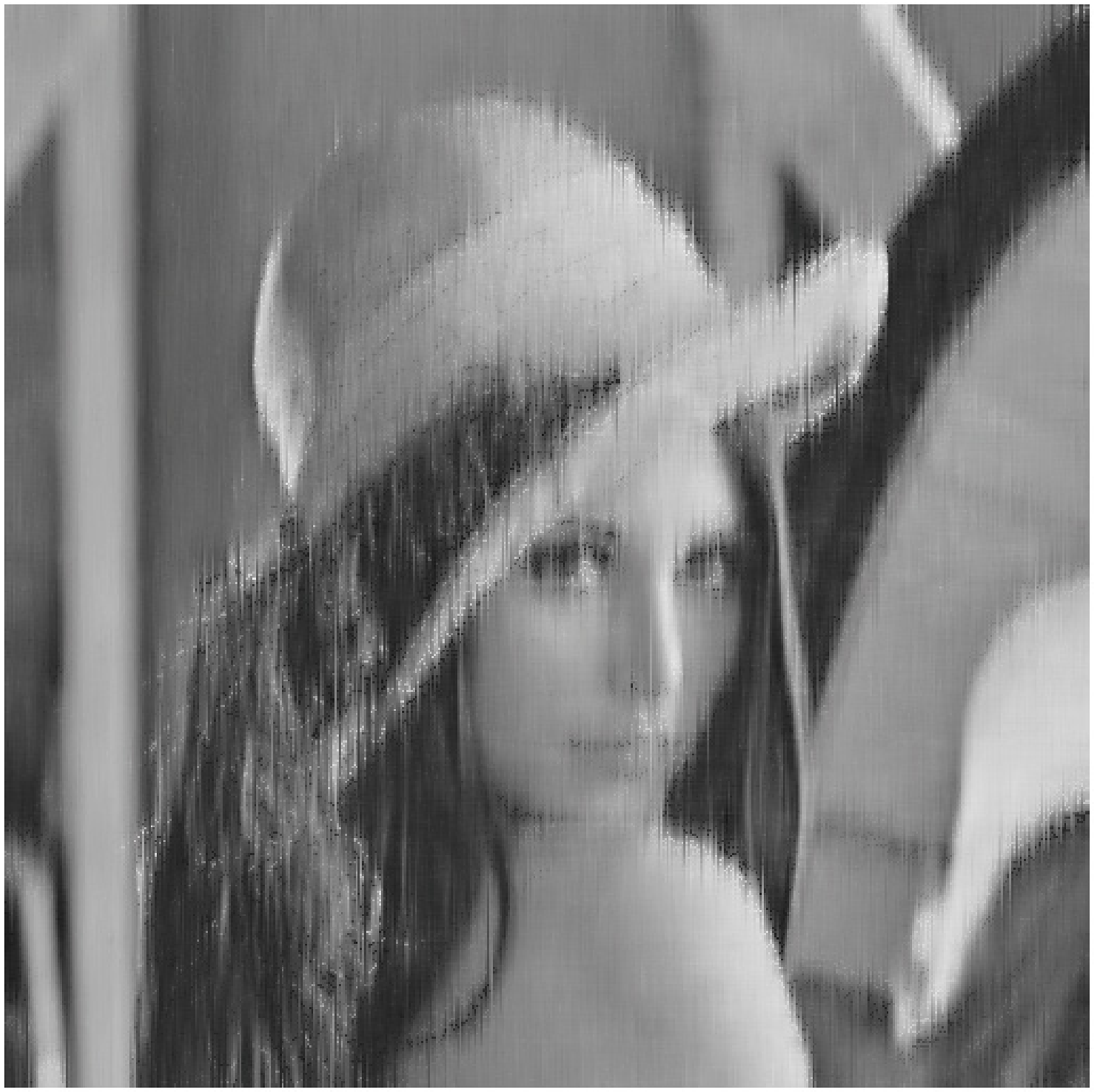}\\
    \includegraphics [width=120pt]{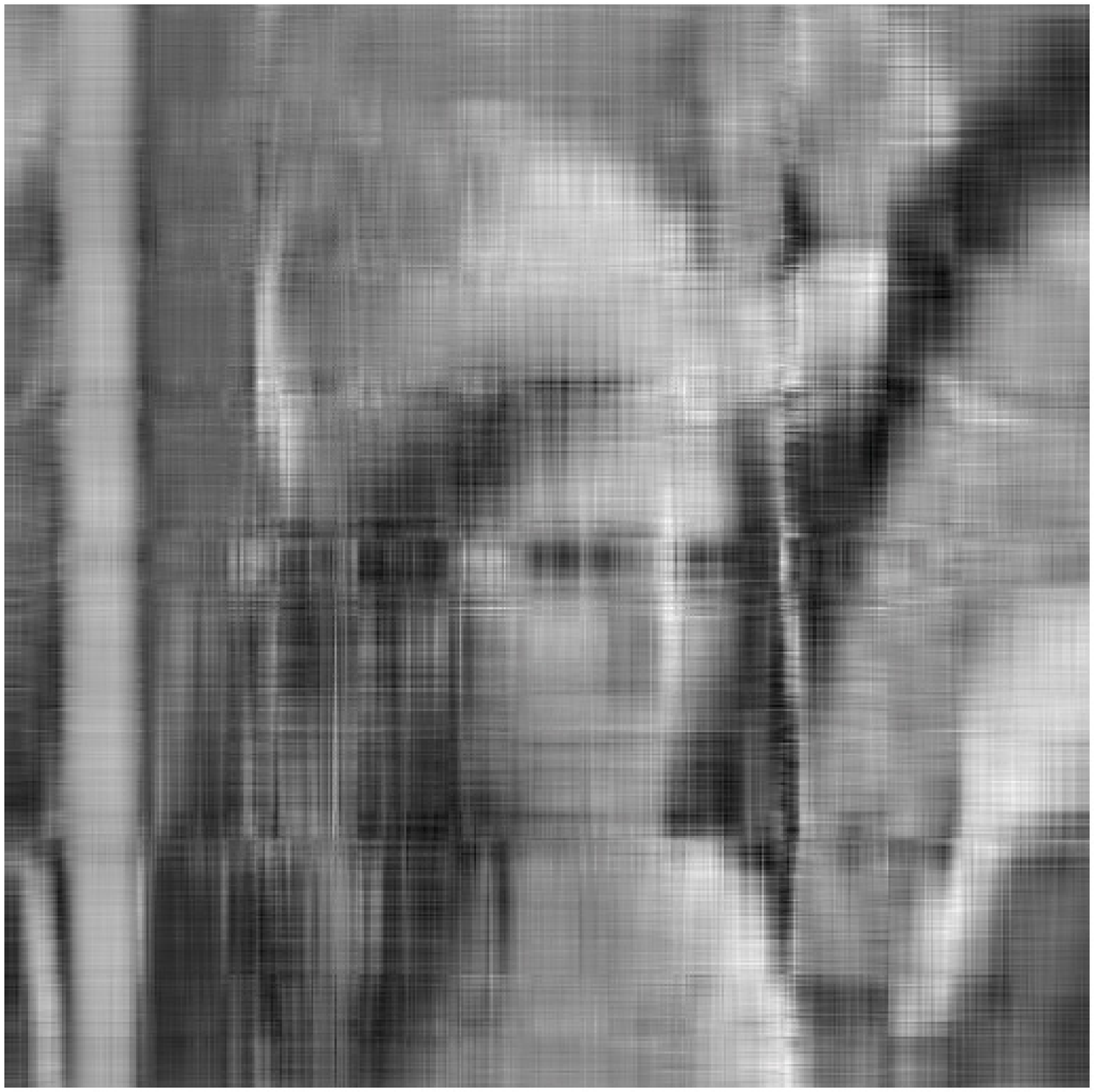}
    \includegraphics [width=120pt]{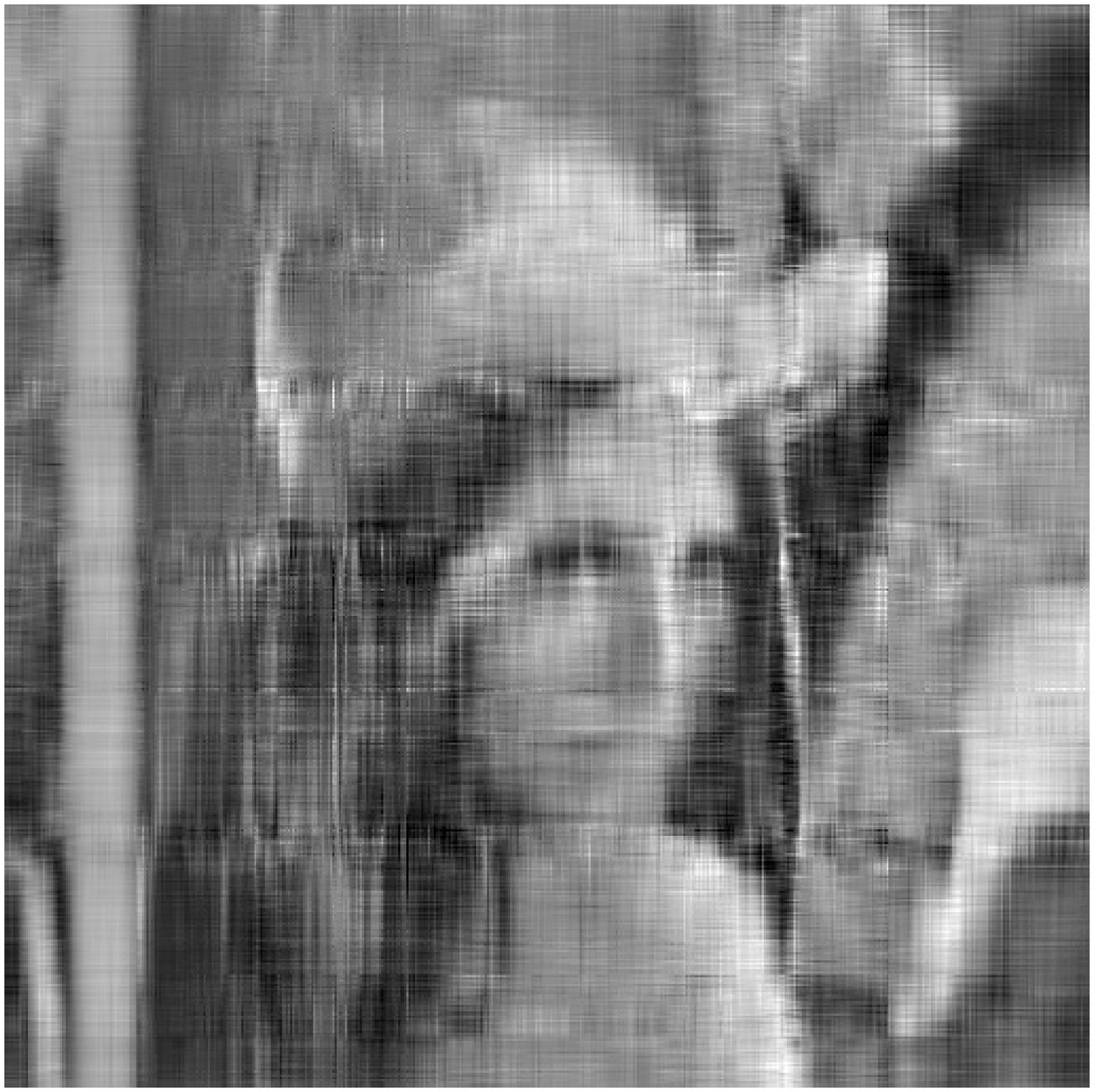}
    \includegraphics [width=120pt]{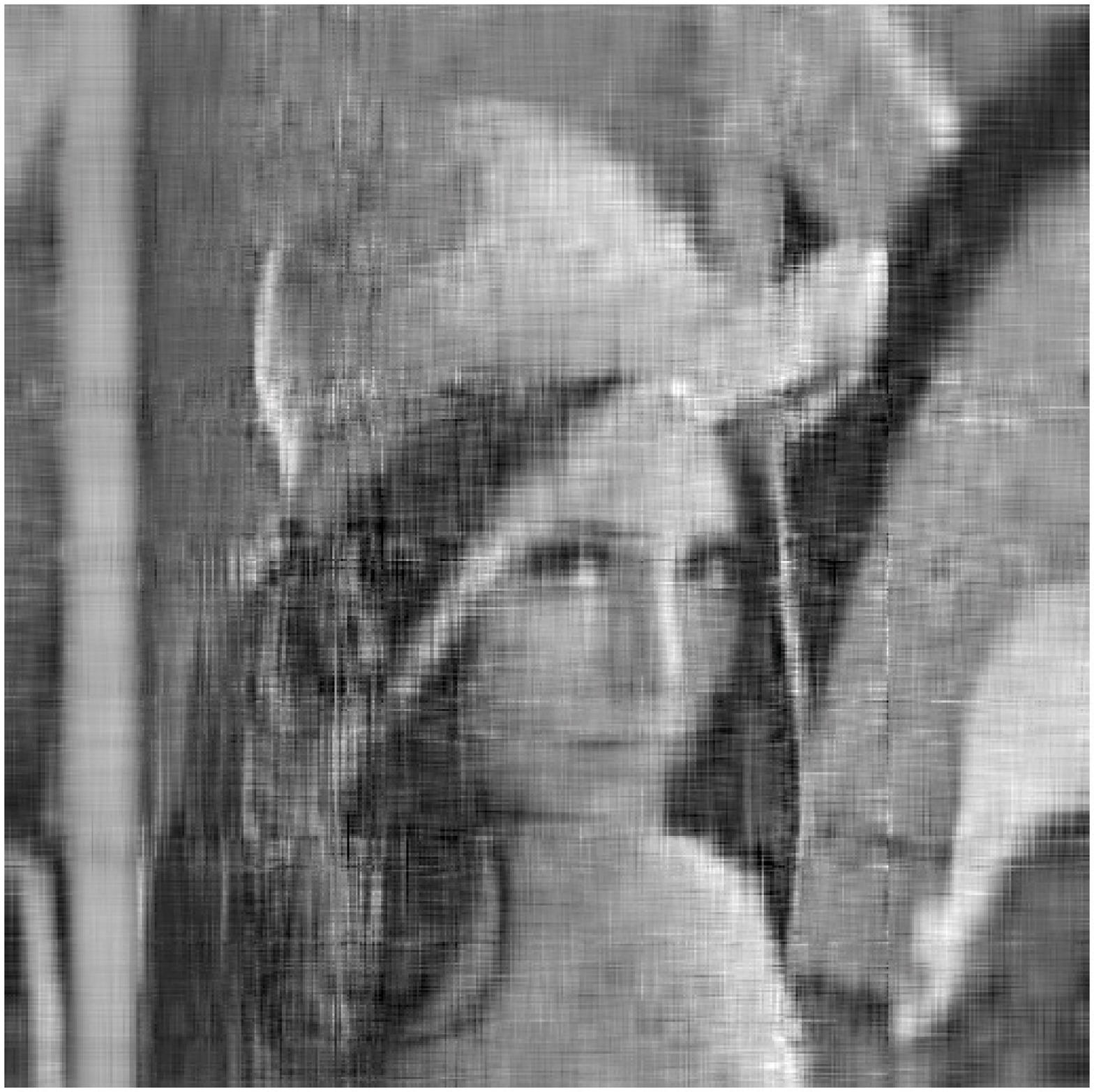}
    \includegraphics [width=120pt]{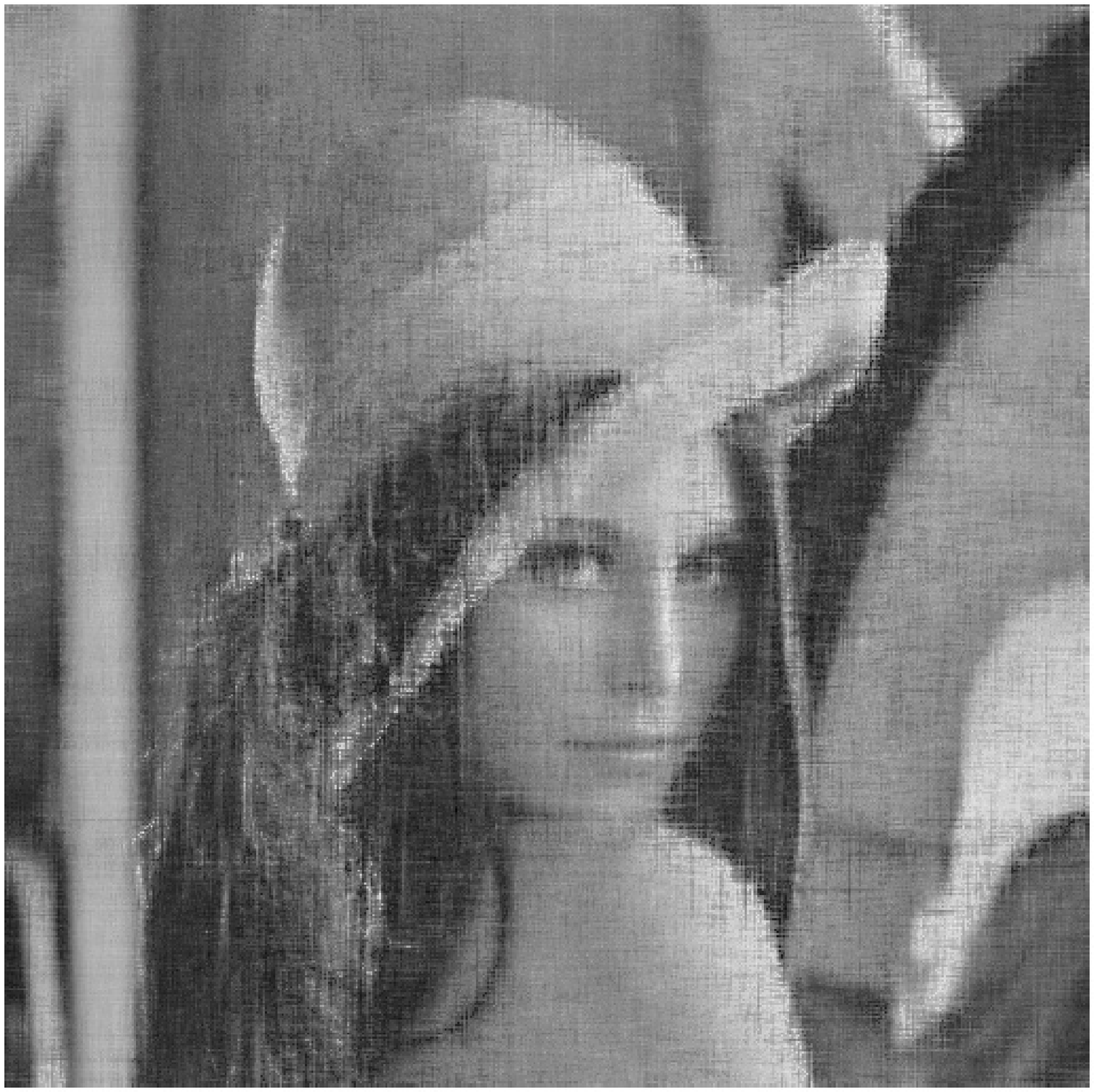}\\
    \caption{
    Top row (from left to right): observed Lena image with missing pixels ($\rho=0.2$),
    images recovered by BMC-GP-GAMP-I,
    BMC-GP-GAMP-II, and BMC-GP-GAMP-III, respectively. Bottom row (from left to right):
    images recovered by VSBL, LMaFit, BiGAMP-MC, and ALM-MC, respectively.}
    \label{fig:lena-02}
\end{figure*}

\begin{figure*}[t]
    \centering
    \includegraphics [width=120pt]{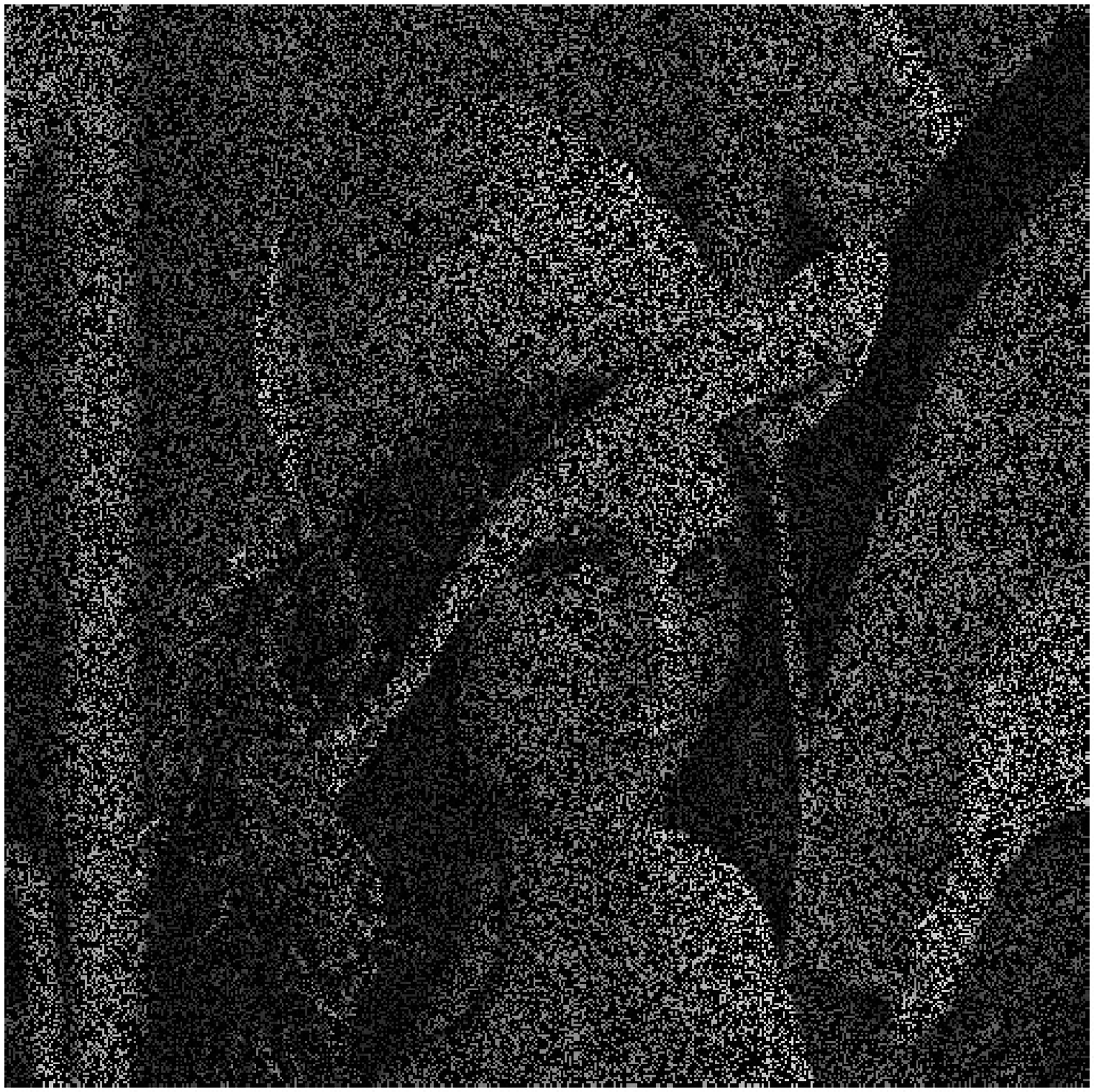}
    \includegraphics [width=120pt]{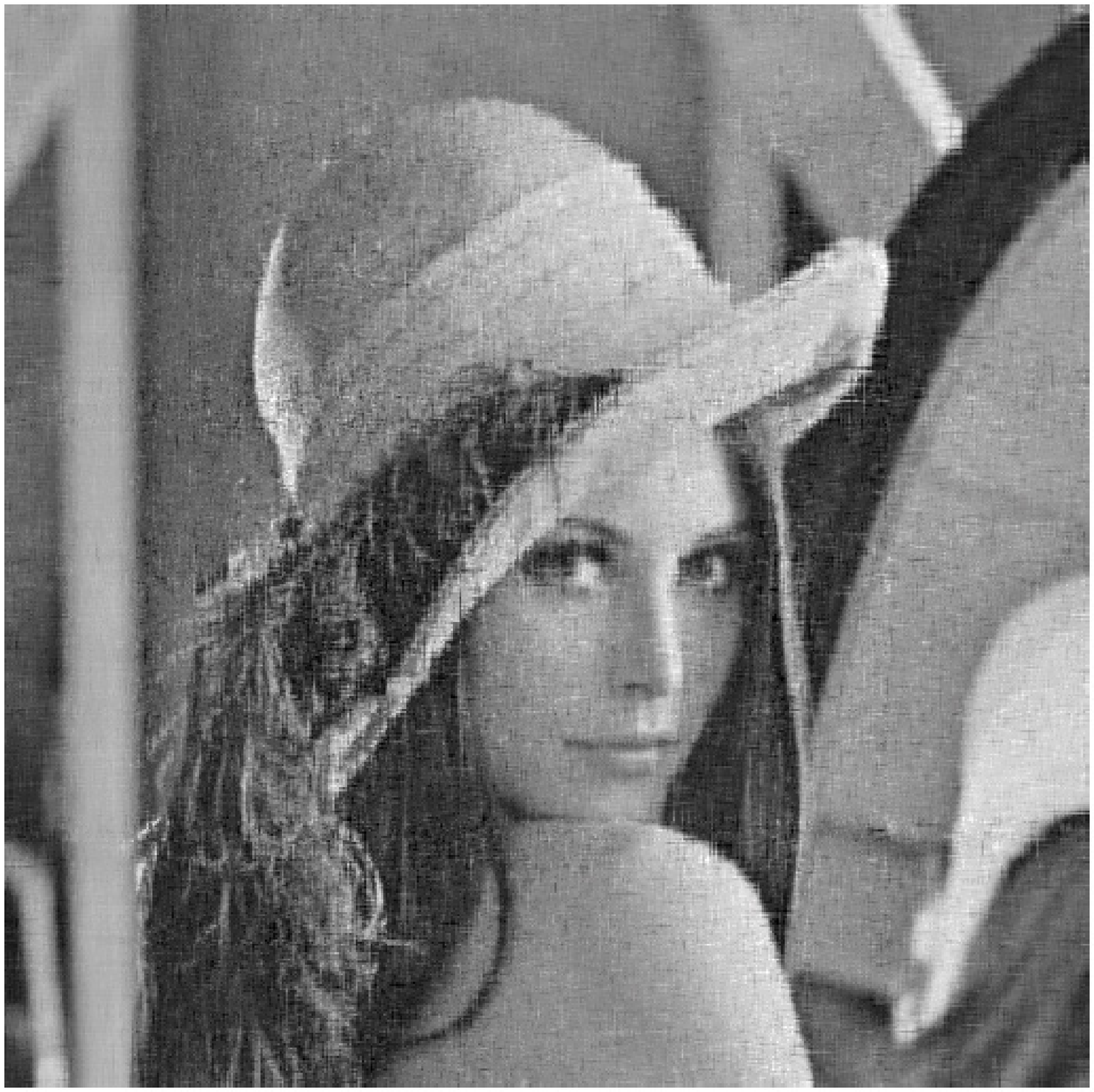}
    \includegraphics [width=120pt]{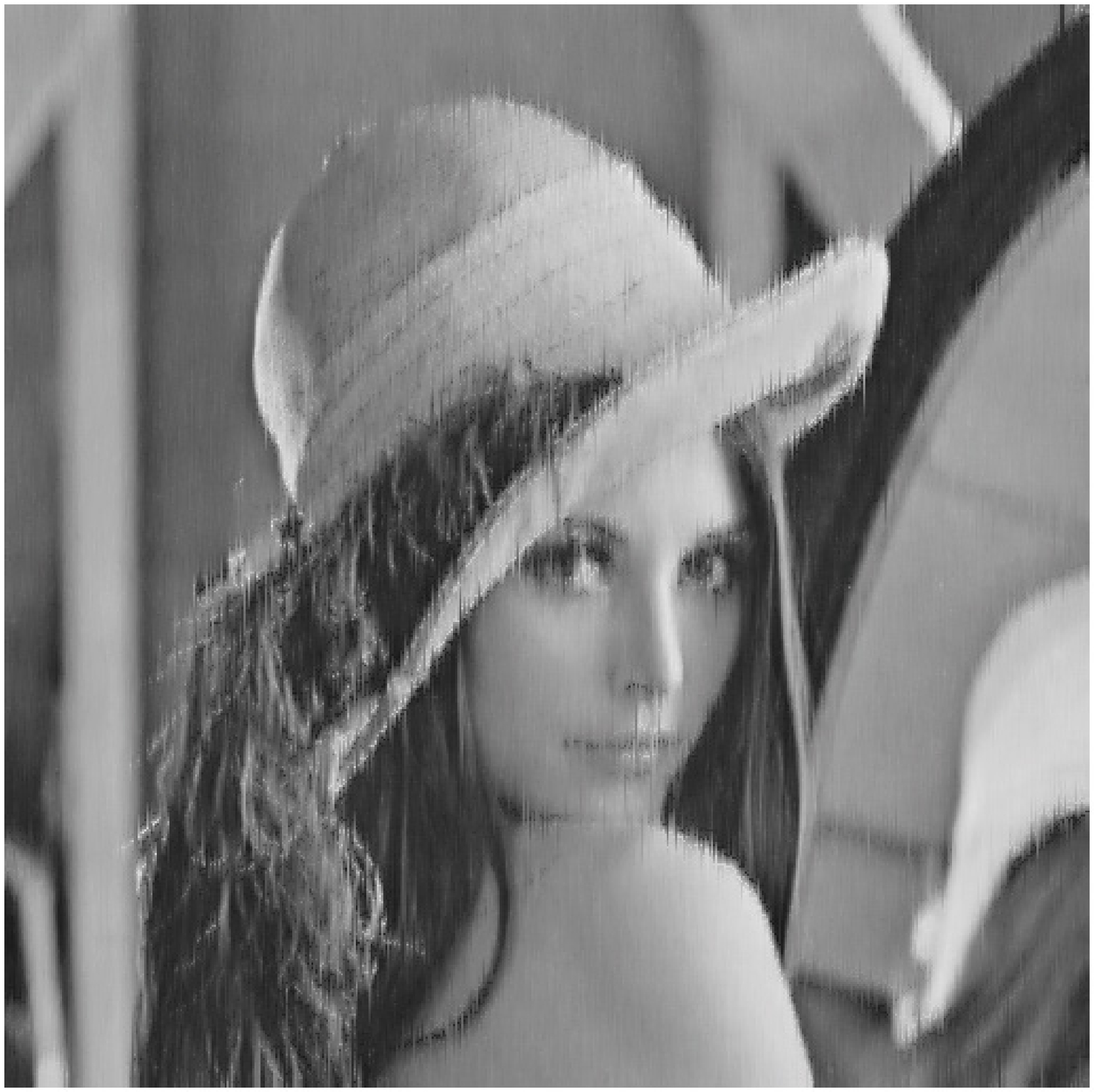}
    \includegraphics [width=120pt]{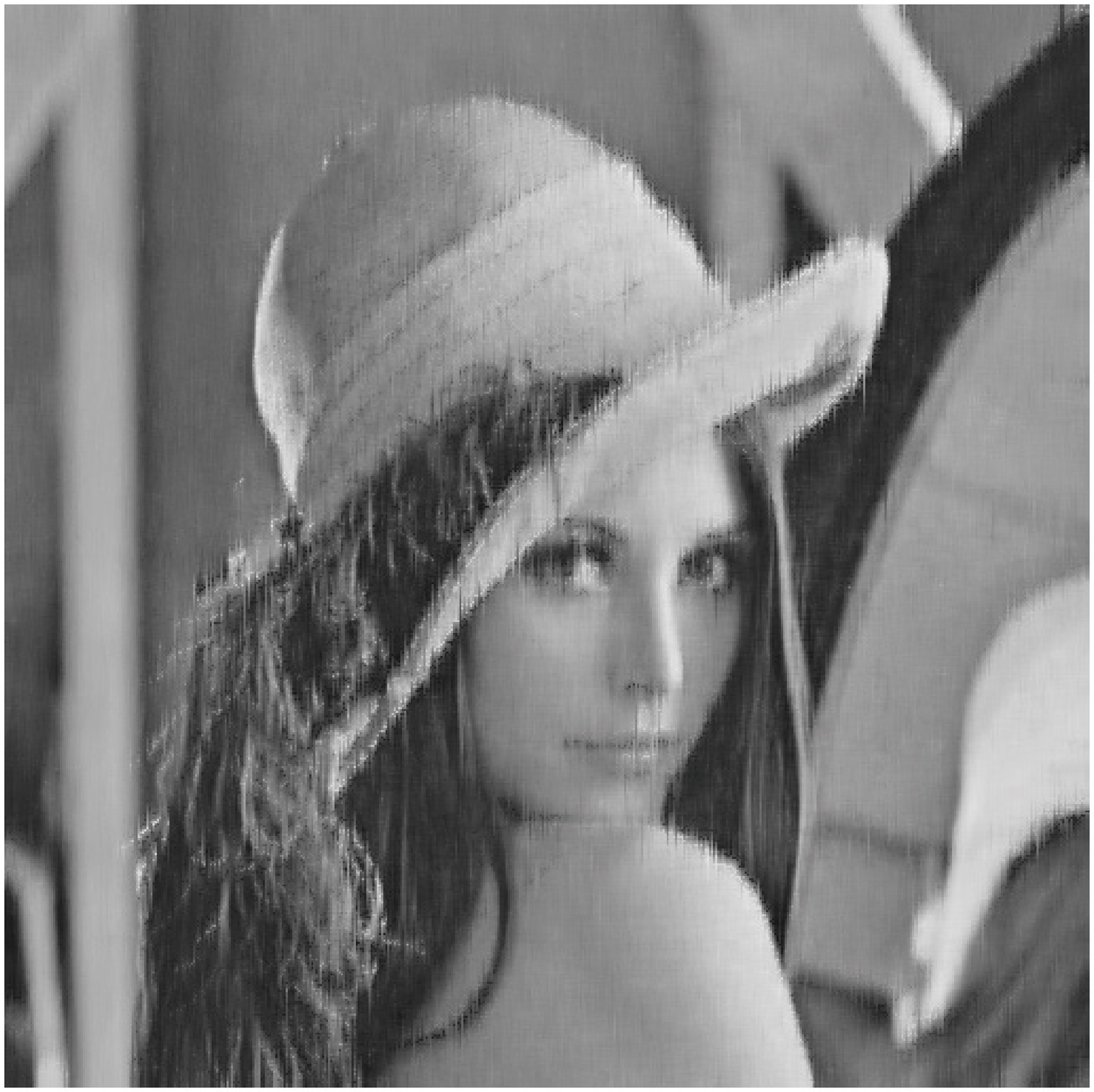}\\
    \includegraphics [width=120pt]{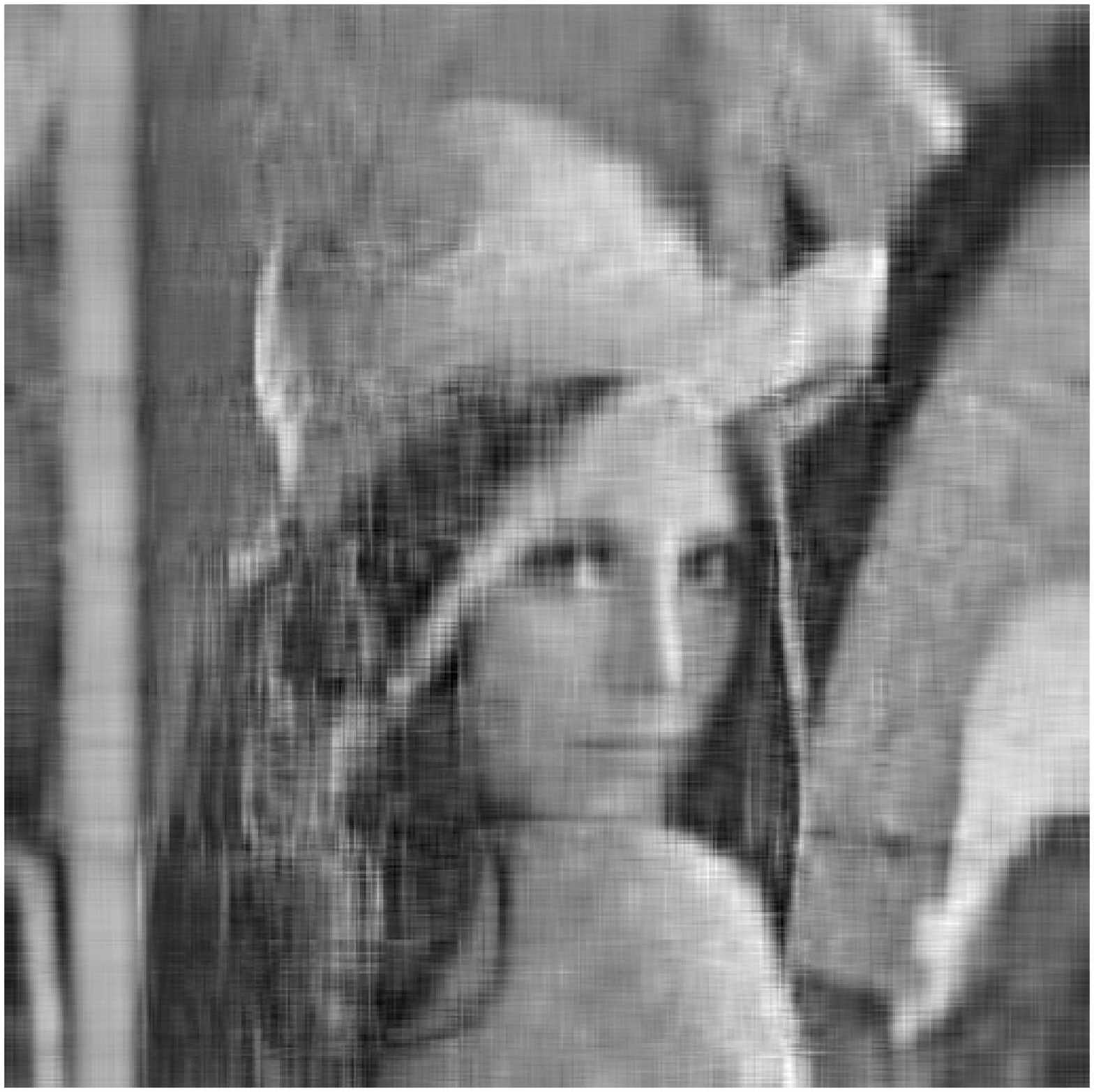}
    \includegraphics [width=120pt]{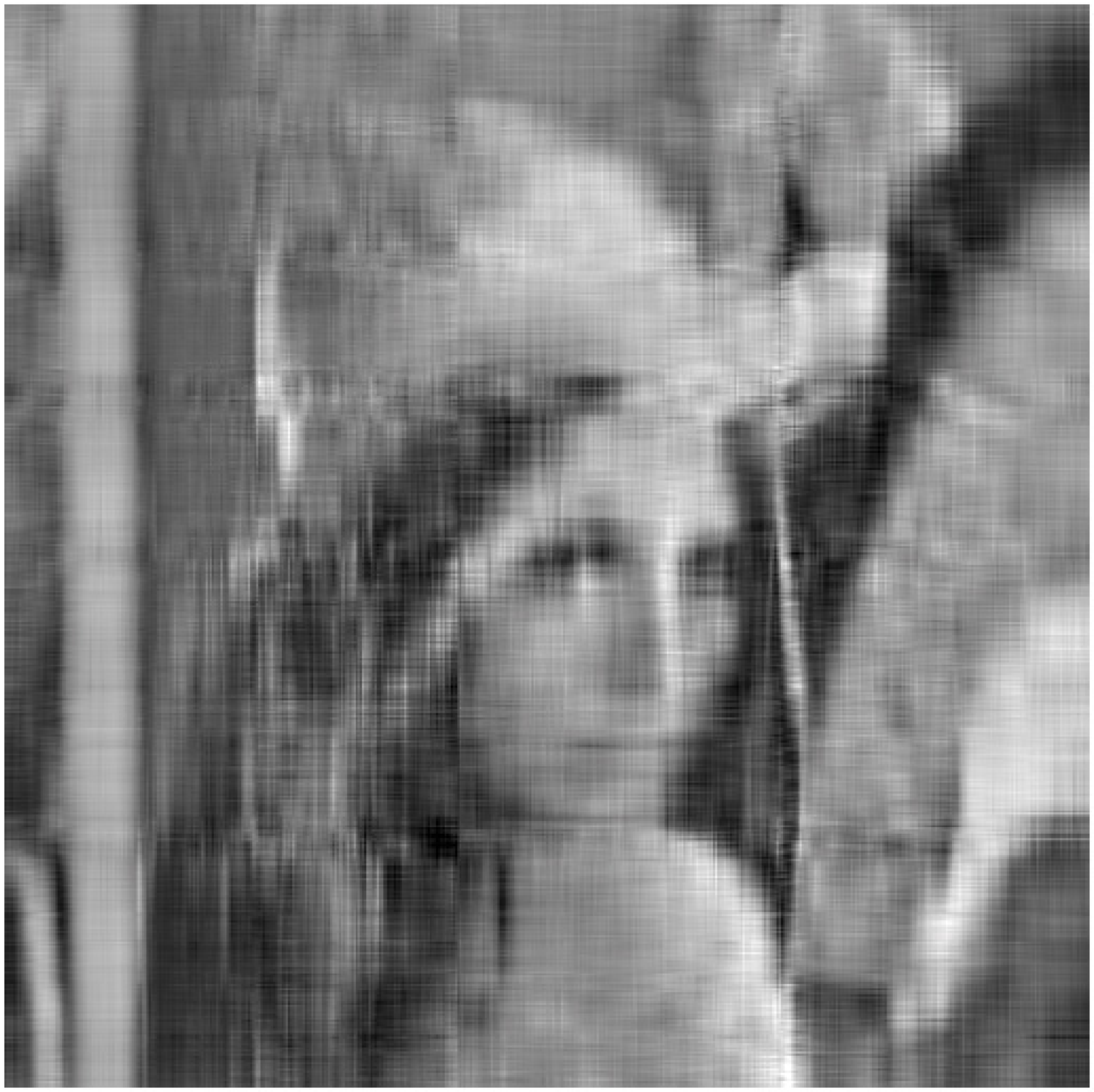}
    \includegraphics [width=120pt]{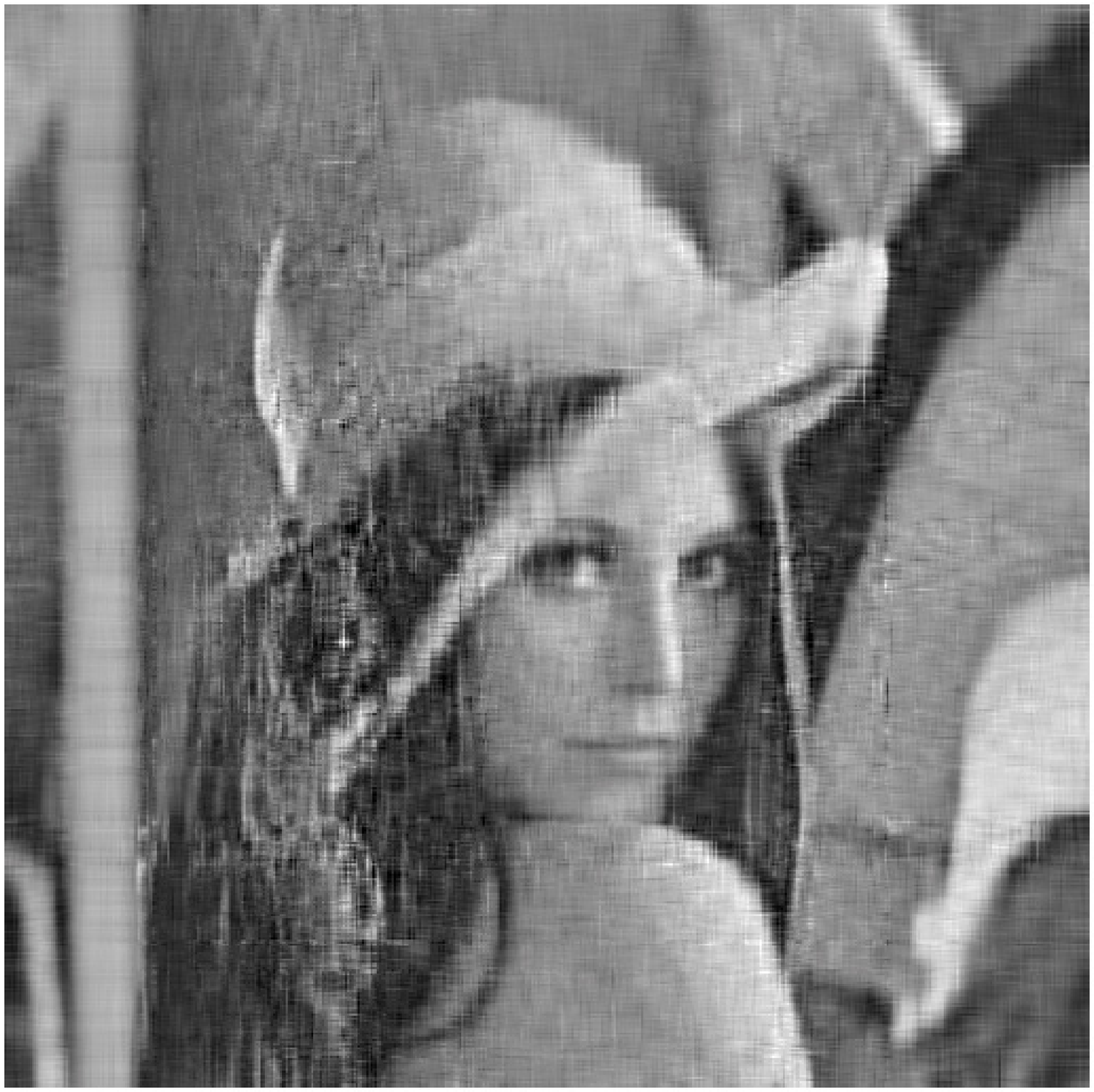}
    \includegraphics [width=120pt]{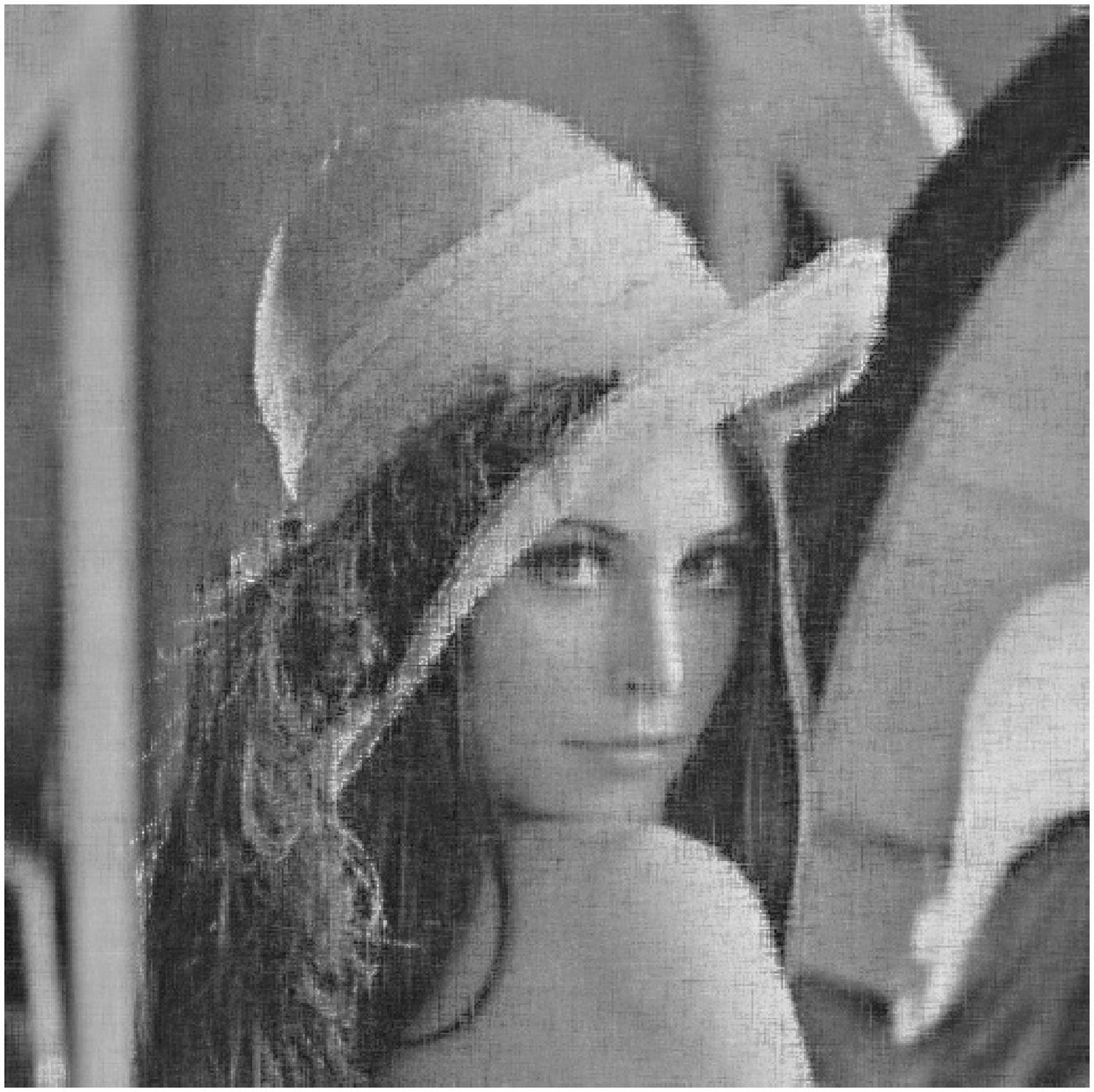}\\
    \caption{
    Top row (from left to right): observed Lena image with missing pixels ($\rho=0.3$), images recovered by BMC-GP-GAMP-I,
    BMC-GP-GAMP-II, and BMC-GP-GAMP-III, respectively. Bottom row (from left to right):
    images recovered by VSBL, LMaFit, BiGAMP-MC, and ALM-MC, respectively.}
    \label{fig:lena-03}
\end{figure*}



\section{Conclusions}\label{sec:conclusions}
The problem of low-rank matrix completion was studied in this
paper. A hierarchical Gaussian prior model was proposed to promote
the low-rank structure of the underlying matrix, in which columns
of the low-rank matrix are assumed to be mutually independent and
follow a common Gaussian distribution with zero mean and a
precision matrix. The precision matrix is treated as a random
parameter, with a Wishart distribution specified as a hyperprior
over it. Based on this hierarchical prior model, we developed a
variational Bayesian method for matrix completion. To avoid
cumbersome matrix inverse operations, the GAMP technique was used
and embedded in the variational Bayesian inference, which resulted
in an efficient VB-GAMP algorithm. Empirical results on synthetic
and real datasets show that our proposed method offers competitive
performance for matrix completion, and meanwhile achieves a
significant reduction in computational complexity.


\useRomanappendicesfalse
\appendices

\section{Detailed Derivation of (\ref{X-marginal})} \label{appB}
We provide a detailed derivation of (\ref{X-marginal}). We have
\begin{align}
p(\boldsymbol{X}) =&\int
\prod\limits_{i=1}^{N}p(\boldsymbol{x}_i|\boldsymbol{\Sigma})p(\boldsymbol{\Sigma})d\boldsymbol{\Sigma}\nonumber\\
\propto&\int\left({\frac{|\boldsymbol{\Sigma}|}{(2\pi)^{M}}}\right)^{\frac{N}{2}}\exp(-\frac{1}{2}
\text{tr}(\boldsymbol{X}^T\boldsymbol{\Sigma}\boldsymbol{X}))\nonumber\\
&\quad\times|\boldsymbol{\Sigma}|^{\frac{\nu-M-1}{2}}\exp(-\frac{1}{2}
\text{tr}(\boldsymbol{W}^{-1}\boldsymbol{\Sigma}))d\boldsymbol{\Sigma}\nonumber
\\
\propto& 2^{\frac{\nu M}{2}}\pi^{-\frac{MN}{2}}
\Gamma_M\left(\frac{\nu+N}{2}\right)|\boldsymbol{W}^{-1}+\boldsymbol{X}\boldsymbol{X}^T|^{-\frac{\nu+N}{2}}\nonumber\\
&\times
\int\frac{|\boldsymbol{\Sigma}|^{\frac{\nu+N-M-1}{2}}\exp(-\frac{1}{2}
\text{Tr}((\boldsymbol{W}^{-1}+\boldsymbol{X}\boldsymbol{X}^T)\boldsymbol{\Sigma}))}
{2^{\frac{(\nu+N)
M}{2}}|(\boldsymbol{W}^{-1}+\boldsymbol{X}\boldsymbol{X}^T)^{-1}|^{\frac{\nu+N}{2}}
\Gamma_M(\frac{\nu+N}{2})}d\boldsymbol{\Sigma} \label{eqn3}
\end{align}
where
\begin{align}
\Gamma_M\left(x\right)=\pi^{\frac{M(M-1)}{4}}\prod_{j=1}^{M}\Gamma\left(x+\frac{1-j}{2}\right)
\end{align}
Note that the term in the integral of (\ref{eqn3}) is a standard
Wishart distribution with $\nu+N$ degrees of freedom and variance
matrix $(\boldsymbol{W}+\boldsymbol{X}\boldsymbol{X}^T)^{-1}$.
Thus we arrive at
\begin{align}
p(\boldsymbol{X})\propto&
    2^{\frac{\nu M}{2}}\pi^{-\frac{MN}{2}}
    \Gamma_M\left(\frac{\nu+N}{2}\right)|\boldsymbol{W}^{-1}+\boldsymbol{X}\boldsymbol{X}^T|^{-\frac{\nu+N}{2}}\nonumber\\
    \propto&|\boldsymbol{W}^{-1}+\boldsymbol{X}\boldsymbol{X}^T|^{-\frac{\nu+N}{2}}
\end{align}

\section{Proof of Lemma \ref{lemma1}} \label{appA}
Since we have
$|\boldsymbol{X}\boldsymbol{X}^T+\boldsymbol{W}^{-1}|=|\boldsymbol{W}^{-1}|
|\boldsymbol{W}\boldsymbol{X}\boldsymbol{X}^T+\boldsymbol{I}|$, we
only need to prove
\begin{align}
|\boldsymbol{W}\boldsymbol{X}\boldsymbol{X}^T+
    \boldsymbol{I}|=|\boldsymbol{X}^T\boldsymbol{W}\boldsymbol{X}+\boldsymbol{I}|
\end{align}

Recalling the determinant of block matrices, we have
\begin{align}
|\boldsymbol{W}\boldsymbol{X}\boldsymbol{X}^T+\boldsymbol{I}|=
\begin{vmatrix}
\boldsymbol{I}&\boldsymbol{X}^T\\
\boldsymbol{0}&\boldsymbol{W}\boldsymbol{X}\boldsymbol{X}^T+\boldsymbol{I}
\end{vmatrix}
\end{align}
and
\begin{align}
|\boldsymbol{I}|=
\begin{vmatrix}
\boldsymbol{I}&\boldsymbol{0}\\
-\boldsymbol{W}\boldsymbol{X}&\boldsymbol{I}
\end{vmatrix}
=
\begin{vmatrix}
\boldsymbol{I}&-\boldsymbol{X}^T\\
\boldsymbol{0}&\boldsymbol{I}
\end{vmatrix}
\end{align}
which yields
\begin{align}
&|\boldsymbol{W}\boldsymbol{X}\boldsymbol{X}^T+\boldsymbol{I}|\nonumber\\
=&
\begin{vmatrix}
\boldsymbol{I}&\boldsymbol{X}^T\\
\boldsymbol{0}&\boldsymbol{W}\boldsymbol{X}\boldsymbol{X}^T+\boldsymbol{I}
\end{vmatrix}\nonumber\\
=&
\begin{vmatrix}
\begin{bmatrix}
\boldsymbol{I}&-\boldsymbol{X}^T\\
\boldsymbol{0}&\boldsymbol{I}
\end{bmatrix}
\begin{bmatrix}
\boldsymbol{I}&\boldsymbol{0}\\
-\boldsymbol{W}\boldsymbol{X}&\boldsymbol{I}
\end{bmatrix}
\begin{bmatrix}
\boldsymbol{I}&\boldsymbol{X}^T\\
\boldsymbol{0}&\boldsymbol{W}\boldsymbol{X}\boldsymbol{X}^T+\boldsymbol{I}
\end{bmatrix}
\end{vmatrix}
\nonumber\\
=&
\begin{vmatrix}
\boldsymbol{X}^T\boldsymbol{W}\boldsymbol{X}+\boldsymbol{I}&\boldsymbol{0}\\
-\boldsymbol{W}\boldsymbol{X}&\boldsymbol{I}
\end{vmatrix}
\nonumber\\
=&|\boldsymbol{X}^T\boldsymbol{W}\boldsymbol{X}+\boldsymbol{I}|
\end{align}
Thus we have
\begin{align}
\log|\boldsymbol{X}\boldsymbol{X}^T+\boldsymbol{W}^{-1}|=\log|\boldsymbol{W}^{-1}|+
\log|\boldsymbol{I}+\boldsymbol{X}^T\boldsymbol{W}\boldsymbol{X}|
\end{align}
This completes the proof.

\bibliography{newbib}
\bibliographystyle{IEEEtran}

\end{document}